\documentclass{article} 
\usepackage{iclr2027_conference,times}

\usepackage{amsmath,amsfonts,bm}

\def\eqref#1{equation~\ref{#1}}

\def\1{\bm{1}}

\DeclareMathAlphabet{\mathsfit}{\encodingdefault}{\sfdefault}{m}{sl}
\SetMathAlphabet{\mathsfit}{bold}{\encodingdefault}{\sfdefault}{bx}{n}

\usepackage{hyperref,wrapfig}
\usepackage{url}
\usepackage{algorithm}
\usepackage{algpseudocode}
\usepackage{amsmath,amsthm}
\usepackage{graphicx}
\usepackage{booktabs}
\usepackage[table]{xcolor} 
\usepackage{grffile} 
\usepackage{multirow}
\usepackage{longtable}
\usepackage{array}
\usepackage{subcaption} 
\newcommand{\vect}[1]{\mathbf{#1}}
\newtheorem{proposition}{Proposition}

\title{EvE: An Alternate Optimizer to Adam}

\author{Shashank Raj \\
Department of Computer Science \\
The University of Alabama \\
Tuscaloosa, AL 35487, USA \\
\texttt{sraj2@crimson.ua.edu} \\
\And
Kalyanmoy Deb \\
Department of Computer Science \\
The University of Alabama \\
Tuscaloosa, AL 35487, USA \\
\texttt{kdeb@ua.edu}
}

\iclrfinalcopy 
\begin{document}

\maketitle

\begin{abstract}
Adam and its variants dominate neural network training, but a single training run only reveals whether one configuration works well after most of its budget is already spent. The concept is a poor fit for hyperparameter or architecture search, where many configurations must be ranked cheaply and pruned early. As an alternative to Adam, we introduce EvE (Evolutionary Explorer), a steady-state, population-of-four differential evolution (DE) optimizer with a targeted Adam fallback: each iteration proposes one candidate via the DE operator. EvE only runs a short burst of gradient descent on it if the DE step fails to improve on the incumbent. Selection is greedy, so on a deterministic objective the best-so-far value is provably monotone non-increasing, and because gradients are used only as a targeted rescue, per-iteration cost stays within a constant factor of a single Adam step regardless of problem dimension. Under a fixed, evaluation-cost-matched budget, EvE wins or ties Adam on $76\%$ of $70$ (problem, dimension) cells across seven classical scalable benchmarks scaling up to one million variables. On three real neural-network training tasks (an MLP on MNIST, and LoRA fine-tuning of a 1.5B-parameter language model on two datasets) EvE consistently finishes the same charged budget $1.7$ to $3.9\times$ faster in wall-clock time, at a modest cost in final quality (about one accuracy point on MNIST, roughly $9$--$11\%$ higher relative test loss on the two language fine-tuning tasks; on GSM8K's full test set, Adam is also about $5$ accuracy points more accurate at matched budgets, and fine-tuning lowers accuracy below the base model for both optimizers). Inside successive halving on UCI Adult Income, EvE completes hyperparameter and architecture searches $3.1$--$3.5\times$ faster, and its ranking of candidate configurations agrees with Adam's about as well as Adam's agrees with itself across training seeds (Kendall's $\tau$ of $0.66$--$0.69$). These results position EvE not as a total replacement for Adam as a final-stage trainer, but as a fast, gradient-aware proxy for the search-heavy, budget-constrained regime one level up: ranking candidate configurations, or pruning clearly bad ones early, where a fast, approximate signal is worth more than a slow, precise one.
\end{abstract}

\vspace{-2mm}
\section{Introduction}
\label{sec:intro}
\vspace{-2mm}

Gradient descent and its adaptive variants are the default engine of deep learning. From stochastic gradient descent \citep{robbins1951stochastic} and momentum \citep{sutskever2013importance} to AdaGrad \citep{duchi2011adaptive}, RMSProp \citep{tieleman2012rmsprop}, Adadelta \citep{zeiler2012adadelta}, and Adam \citep{kingma2015adam}, these methods succeed because they are cheap: one backward pass and a few scalar moment estimates suffice to adapt a per-coordinate step. This comes with two costs. First, first-order methods see only the local slope, so on landscapes with many local optima, saddles, or plateaus, gradients alone need not lead to a good solution \citep{abdulkadirov2023survey}, which is why the field keeps producing variants that reshape the moment estimates \citep{loshchilov2019decoupled,liu2020radam,dozat2016nadam,dubey2020diffgrad,zaheer2018adaptive,zhuang2020adabelief,luo2019adabound,xie2021positive,xie2022adan}, add curvature \citep{liu1989limited,ma2021apollo,yao2021adahessian}, or use information geometry \citep{martens2020new,nielsen2020elementary}, all still refining a single point from local, differentiable information. Second, Adam typically needs many epochs before a training curve is informative, acceptable for one converged model but a poor fit for ranking many candidate configurations cheaply, as in hyperparameter and architecture search, where poor configurations should be pruned early \citep{bergstra2011algorithms,elsken2019neural}. Population-based, gradient-free methods such as differential evolution \citep{storn1997differential}, evolution strategies \citep{hansen2001completely,salimans2017evolution}, and particle swarm optimization \citep{kennedy1995particle} explore several regions at once without a differentiability requirement, but waste evaluations rediscovering directions a single backward pass would reveal.

This tension motivates a hybrid, not a replacement. EvE is built around a steady-state population of four individuals, each updated by a differential evolution step; when this fails to improve an individual, EvE runs a short burst of Adam on it alone, using its gradient at that point (a cheap, gated polynomial mutation is also applied before the DE child is checked on the synthetic benchmarks, but not for neural network training; Section~\ref{sec:mutation}). Selection is greedy, so the best-so-far value is monotonically non-increasing by construction (Section~\ref{sec:theory}), and because gradients are only a targeted rescue, per-iteration cost stays within a constant factor of a single Adam optimizer regardless of dimension. Under this budget, EvE wins or matches Adam on most of the classical scalable benchmarks tested, up to $10^6$ variables, and finishes real neural-network training tasks $1.7$--$3.9\times$ faster at a modest cost in final quality (Section~\ref{sec:experiments}). That trade-off is this paper's central finding: a fast, gradient-aware proxy for the search-heavy step before training a single model, not a replacement for Adam as a final-stage trainer.

\vspace{-2mm}
\section{Related Work}
\label{sec:related}
\vspace{-2mm}

\paragraph{Gradient-based optimization.} Most neural network training uses first-order stochastic optimization with adaptive per-coordinate step sizes. AdaGrad \citep{duchi2011adaptive} accumulates squared gradients, RMSProp \citep{tieleman2012rmsprop} and Adadelta \citep{zeiler2012adadelta} replace that with an exponential moving average, and Adam \citep{kingma2015adam} combines bias-corrected first and second moments. Later work targets specific failure modes of Adam: decoupled weight decay \citep{loshchilov2019decoupled}, early-training variance rectification \citep{liu2020radam}, Nesterov-style lookahead \citep{dozat2016nadam}, reshaped or bounded second moments \citep{dubey2020diffgrad,zaheer2018adaptive,zhuang2020adabelief,luo2019adabound}, and extensions to the moment structure itself \citep{xie2021positive,xie2022adan}, alongside quasi-Newton \citep{liu1989limited,ma2021apollo,yao2021adahessian} and information-geometric \citep{martens2020new,nielsen2020elementary} methods (surveyed across architecture families by \citet{abdulkadirov2023survey}). All of these methods refine a single point with local differentiable information, the constraint EvE relaxes.

\paragraph{Evolutionary and population-based optimization.} Differential evolution \citep{storn1997differential} generates candidates from scaled differences between population members, followed by crossover and greedy selection. Evolution strategies \citep{hansen2001completely,salimans2017evolution} adapt a search distribution over generations, and particle swarm optimization \citep{kennedy1995particle} pulls each particle toward personal and global bests. Real-coded genetic algorithms contribute the polynomial mutation operator of \citet{deb1996combined}, standardized in NSGA-II \citep{deb2002fast}, which EvE reuses. These methods need no gradient and resist collapsing into the first local optimum, but sample efficiency is poor as dimension grows \citep{such2017deep}, so they are used mostly for hyperparameter and architecture search \citep{bergstra2011algorithms,elsken2019neural} rather than training large networks directly, though they have proven effective well beyond training itself \citep{raj2026evosort,raj2026hybmsearch,deb2026directly,guha2026regularity,raj2026balance}.

\paragraph{Hybrid methods.} Prior hybrids typically use a gradient-free method to explore globally and a gradient step as a separate polishing phase on the whole population or its best member \citep{he2022hybrid,abdulkadirov2023survey}. EvE couples the two more tightly: the Adam fallback is a per-individual, per-iteration rescue triggered only when mutation fails, so gradient cost scales with how often mutation stalls, not with population size.

\vspace{-2mm}
\section{Method}
\label{sec:method}
\vspace{-2mm}

Polynomial mutation is used on the low-dimensional synthetic benchmarks of Section~\ref{sec:experiments} but off by default for neural network training; Section~\ref{sec:mutation} explains why.

\subsection{Setup}
\label{sec:setup}

We minimize $f:\mathbb{R}^{n}\to\mathbb{R}$ over a box $[\vect{x}_l,\vect{x}_u]$, with $f$ possibly non-convex and not required to be differentiable everywhere. All internal search happens on the unit cube $z\in[0,1]^n$ via the affine map $x=\vect{x}_l+z\odot(\vect{x}_u-\vect{x}_l)$, so a fixed-size step in $z$ never overshoots narrow-range coordinates or under-moves wide-range ones. The population at iteration $t$ is $Z^{(t)}=\{z_1,\dots,z_4\}$ with values $f_i^{(t)}=f(\vect{x}_l+z_i^{(t)}\odot(\vect{x}_u-\vect{x}_l))$; we fix $n_{\text{pop}}=4$, keeping a few competing candidates open for the Adam fallback to polish, not to cover the space.

\subsection{Proximity initialization}
\label{sec:init}

Uniform initialization of four points gives difference vectors $z_b-z_c$ that are themselves near-uniform, carrying no local information. Instead we draw one anchor $z_1\sim\mathcal{U}(0,1)^n$ and place the other three members as independent Gaussian jitters around it,
\begin{equation}
\label{eq:proximity}
z_k=\text{clamp}\big(z_1+\sigma\,\epsilon_k,\,0,1\big),\qquad \epsilon_k\sim\mathcal{N}(0,I_n),\qquad k=2,3,4,
\end{equation}
with $\sigma=0.1$, so the population is a compact neighborhood from generation one and DE differences and the pbest pull (Section~\ref{sec:de}) are informative immediately. This is the default for neural network training (Section~\ref{sec:mnist}); the synthetic benchmark suite instead uses Latin Hypercube Sampling, spreading all four points across $[0,1]^n$ at once rather than clustering them (Appendix~\ref{app:benchmarks} reports the comparison). 

\subsection{Steady-state differential evolution}
\label{sec:de}

Each iteration challenges one slot $i$ with one offspring, matching the Adam fallback's granularity and keeping per-iteration cost independent of $n_{\text{pop}}$; one generation is $n_{\text{pop}}$ iterations. Drawing three distinct donors $a,b,c\neq i$ uniformly from the other three slots, we use a rand-to-pbest/1 DE update \citep{storn1997differential}:
\begin{equation}
\label{eq:de}
z_{\text{child}}=z_a+F_t\,(z_b-z_c)+F\,(z_{\text{best}}-z_a),\qquad z_{\text{best}}=\arg\min_{z\in Z^{(t)}}f(z),
\end{equation}
with base weight $F=0.5$ and a per-iteration multiplicative jitter $F_t=F\,(1+\gamma(u-0.5))$, $u\sim\mathcal{U}(0,1)$, $\gamma=10^{-4}$, applied only to the difference term (the pbest pull uses the unjittered $F$), decorrelating the difference-vector scale across iterations without changing its mean. The pbest term, absent from the original rand/1 scheme, biases every offspring toward the best known point immediately, since three donors alone need many iterations to find an improving direction. The offspring is clamped to $[0,1]^n$; slot $i$ only enters later, checked against $f_i^{(t)}$ (Section~\ref{sec:selection}), never in construction.

\subsection{Polynomial mutation (PM)}
\label{sec:mutation}

On the synthetic benchmarks of Section~\ref{sec:experiments} the DE offspring is perturbed with the polynomial mutation (PM) operator of \citet{deb1996combined,deb2002fast}, guarded by its own accept/reject gate: PM produces a mutant of $z_{\text{child}}$, both are scored on the same batch, and $z_{\text{child}}$ is replaced by the mutant only if the mutant is strictly better; otherwise the clean, unmutated child is kept. This local gate means PM can only ever leave $z_{\text{child}}$ unchanged or improve it, at the cost of one extra evaluation, before Section~\ref{sec:adam-fallback}'s Adam fallback and Section~\ref{sec:selection}'s selection ever see it. For a selected gene $z_d$ with bounds $[z_l,z_u]_d$, PM draws $r\sim\mathcal{U}(0,1)$ and sets $z_d\gets z_d+\delta_q\,(z_u-z_l)_d$ with
\begin{equation}
\label{eq:pm}
\delta_q=
\begin{cases}
\big[\,2r+(1-2r)(1-\delta_1)^{\eta+1}\,\big]^{\frac{1}{\eta+1}}-1, & r\le 0.5,\\[3pt]
1-\big[\,2(1-r)+2(r-0.5)(1-\delta_2)^{\eta+1}\,\big]^{\frac{1}{\eta+1}}, & r>0.5,
\end{cases}
\end{equation}
where $\delta_1=(z_d-z_l)/(z_u-z_l)$, $\delta_2=(z_u-z_d)/(z_u-z_l)$, and $\eta=20$ is the distribution index: larger $\eta$ concentrates $\delta_q$ near $0$ (parent, $z_{\text{child}}$). PM's perturbation is bounded and self-scaling to the distance to each boundary, needing no clipping and wasting no probability on infeasible values -- why NSGA-II adopts it. Each individual mutates $k=\max\!\big(1,\text{round}(n\cdot p)\big)=1$ gene, the classical per-gene rate $p=1/n$: exactly one coordinate per individual regardless of $n$, worth doing only because the accept/reject gate above makes the attempt risk-free at the cost of one extra evaluation.

\paragraph{Why PM helps here but not on neural network training?} We turn PM off by default when training a real network: the gate means it can never hurt, but at millions of parameters a single-coordinate nudge essentially never wins it, so it becomes a pure cost with no measurable benefit. Appendix~\ref{app:pm} gives the full argument.

\subsection{Adam fallback}
\label{sec:adam-fallback}

A failed candidate should not waste its evaluation: when the (possibly PM's) child fails to beat the target slot's value, $f(z_{\text{child}})\ge f_i^{(t)}$, EvE runs $n_{\text{adam}}$ Adam \citep{kingma2015adam} steps on it, in real coordinates $x=\vect{x}_l+z_{\text{child}}\odot(\vect{x}_u-\vect{x}_l)$, not the normalized $z$-space the rest of the method uses: with $g_t=\nabla_x f(x_{t-1})+\lambda x_{t-1}$:
\begin{align}
m_t&=\beta_1 m_{t-1}+(1-\beta_1)g_t, & v_t&=\beta_2 v_{t-1}+(1-\beta_2)g_t^2,\\
\hat m_t&=m_t/(1-\beta_1^{t}), & \hat v_t&=v_t/(1-\beta_2^{t}),\\
x_t&=x_{t-1}-\eta_{\text{lr}}\,\hat m_t\big/(\sqrt{\hat v_t}+\epsilon), \quad x_t \gets\text{clamp}(x_t,\vect{x}_l,\vect{x}_u), \label{eq:adamw}
\end{align}
with default $\beta_1=0.9$, $\beta_2=0.999$, $\epsilon=10^{-8}$; weight decay $\lambda$ enters Equation~\ref{eq:adamw} as a classic L2 penalty folded directly into the gradient before the moment update (PyTorch's stock \texttt{Adam}), not AdamW's decoupled rule \citep{loshchilov2019decoupled}, and the result is clamped back into the real bounds. Each population slot keeps its own moment state $(m,v,t)$ across the whole run, so fallback bursts form one continuous Adam trajectory rather than repeated cold starts; only after the burst is $x_{n_{\text{adam}}}$ mapped back to normalized coordinates. We use Adam rather than SGD or L-BFGS \citep{liu1989limited} since the fallback runs on a stochastic minibatch objective, where per-coordinate second-moment normalization handles gradient noise across differently scaled weight tensors. The refined point replaces $z_{\text{child}}$ only if it beats $f(z_{\text{child}})$, so the fallback never hurts a slot (Section~\ref{sec:monotone}). The fallback slot is not tied to Adam specifically: any local, gradient-based optimizer could occupy it under the same accept-only-if-improving rule; we use Adam throughout and do not evaluate other choices.

\subsection{Selection and common random numbers}
\label{sec:selection}

The winner of the iteration, $z_{\text{cand}}$, is $z_{\text{child}}$ (after PM, if active) unless the Adam fallback fired and beat it, in which case $z_{\text{cand}}$ is the Adam-refined point. $z_{\text{cand}}$ replaces slot $i$ iff it strictly beats $f_i^{(t)}$; otherwise the slot is untouched.

\paragraph{How the batch advances.} When $f$ is a minibatch loss, one new minibatch $\mathcal{B}_t$ is drawn once per iteration; every evaluation that iteration, including the whole fallback burst, is scored on it (common random numbers), so the greedy comparison is not confounded by sampling noise. Appendix~\ref{app:batch} gives full details.
Algorithm~\ref{alg:spade} gives the complete steady-state loop.
\begin{algorithm}[t]
\caption{EvE}
\label{alg:spade}
\scalebox{0.65}{%
\begin{minipage}{\linewidth}
\begin{algorithmic}[1]
\Require $f$, bounds $\vect{x}_l,\vect{x}_u$, iterations $T=4n_{\text{gen}}$, $F$, $\sigma$, $\eta$, $n_{\text{adam}}$, $\eta_{\text{lr}}$, $\lambda$, seed $s$, \textsc{usePM}
\State $z_1\sim\mathcal{U}(0,1)^n$;\quad $z_k\gets\text{clamp}(z_1+\sigma\,\epsilon_k,0,1)$, $\epsilon_k\sim\mathcal{N}(0,I_n)$, for $k=2,3,4$
\State $f_k\gets f(\vect{x}_l+z_k\odot(\vect{x}_u-\vect{x}_l))$ for $k=1,\dots,4$
\State $\text{state}_k\gets(\vect{0},\vect{0},0)$ for $k=1,\dots,4$ \Comment{persistent Adam moments, one per slot}
\For{$t=1,\dots,T$}
    \State $i\gets ((t-1)\bmod 4)+1$ \Comment{round-robin target slot}
    \State \textbf{if} $f$ stochastic \textbf{then} draw new minibatch $\mathcal{B}_t$; $f_i\gets f(x_i;\mathcal{B}_t)$
    \State draw distinct $a,b,c\in\{1,\dots,4\}\setminus\{i\}$; draw $u\sim\mathcal{U}(0,1)$; $F_t\gets F(1+\gamma(u-0.5))$
    \State $z_{\text{child}}\gets\text{clamp}\big(z_a+F_t(z_b-z_c)+F(z_{\arg\min_k f_k}-z_a),\,0,1\big)$
    \State $f_{\text{child}}\gets f(z_{\text{child}};\mathcal{B}_t)$
    \If{\textsc{usePM}} \Comment{PM's own accept/reject gate: mutant only kept if it wins}
        \State $z_{\text{mut}}\gets\textsc{PolynomialMutation}(z_{\text{child}},\eta)$;\quad $f_{\text{mut}}\gets f(z_{\text{mut}};\mathcal{B}_t)$
        \State \textbf{if} $f_{\text{mut}} < f_{\text{child}}$ \textbf{then} $z_{\text{child}},f_{\text{child}}\gets z_{\text{mut}},f_{\text{mut}}$
    \EndIf
    \If{$f_{\text{child}} < f_i$}
        \State $z_{\text{cand}},f_{\text{cand}}\gets z_{\text{child}},f_{\text{child}}$
    \Else
        \State $x\gets \vect{x}_l+z_{\text{child}}\odot(\vect{x}_u-\vect{x}_l)$
        \For{$n=1,\dots,n_{\text{adam}}$}
            \State $x\gets \textsc{AdamStep}\big(x,\nabla_x f(x;\mathcal{B}_t),\eta_{\text{lr}},\lambda,\text{state}_i\big)$;\quad $x\gets\text{clamp}(x,\vect{x}_l,\vect{x}_u)$
        \EndFor
        \State $z_{\text{adam}}\gets(x-\vect{x}_l)/(\vect{x}_u-\vect{x}_l)$;\quad $f_{\text{adam}}\gets f(z_{\text{adam}};\mathcal{B}_t)$
        \State $(z_{\text{cand}},f_{\text{cand}})\gets\arg\min_{(z,f)\in\{(z_{\text{child}},f_{\text{child}}),(z_{\text{adam}},f_{\text{adam}})\}} f$
    \EndIf
    \State \textbf{if} $f_{\text{cand}} < f_i$ \textbf{then} $z_i\gets z_{\text{cand}}$; $f_i\gets f_{\text{cand}}$
\EndFor
\State \Return $\arg\min_k z_k,\ \min_k f_k$
\end{algorithmic}
\end{minipage}}
\end{algorithm}
Appendix~\ref{app:algo} gives implementation and memory-cost details, the full batch-advance discussion, and the extended argument for why PM helps the benchmarks but not neural network training.

\vspace{-2mm}
\section{Theoretical Properties}
\label{sec:theory}
\vspace{-1mm}

We examine two theoretical properties of EvE: that its best-so-far objective value never worsens across iterations, and that its per-iteration cost stays within a constant factor of a single Adam step regardless of problem dimension $n$.

\paragraph{Monotone best-so-far.}
\label{sec:monotone}
Let $f_{\text{best}}^{(t)}=\min_i f_i^{(t)}$.
\begin{proposition}
For deterministic $f$, $f_{\text{best}}^{(t+1)}\le f_{\text{best}}^{(t)}$\ for every $t$.
\end{proposition}
\begin{proof}
Iteration $t{+}1$ touches only slot $i$, and greedy selection gives $f_i^{(t+1)}=\min(f_i^{(t)},f(z_{\text{cand}}))\le f_i^{(t)}$; all other slots are unchanged. Hence $\min_j f_j^{(t+1)}\le\min_j f_j^{(t)}$.
\end{proof}
This holds for any $F,\sigma,\eta,n_{\text{adam}}$ and whether or not the fallback fires, since DE, PM, and Adam only ever propose a candidate that is checked against the incumbent before acceptance. For a stochastic minibatch $f$, the guarantee is weaker: Algorithm~\ref{alg:spade} (line 6) refreshes the challenged slot on a freshly drawn batch before comparison, so $f_i$ can move from sampling noise alone; what is guaranteed is only that a candidate is never accepted unless it beats the incumbent on the same batch (common random numbers, Section~\ref{sec:selection}), so Section~\ref{sec:experiments}'s training curves are not literally non-increasing step to step.

\paragraph{Per-iteration cost.}
Let $C_f$ be the cost of one forward pass and assume one gradient costs $c_\nabla C_f$ with $c_\nabla=O(1)$ (reverse-mode AD). One iteration does $O(n)$ elementwise work for DE, plus a constant number of forward evaluations and, when the fallback fires, $n_{\text{adam}}$ forward-backward passes. The worst case (fallback every iteration) is $O(n)+\big(3+n_{\text{adam}}(1+c_\nabla)\big)C_f=O(n_{\text{adam}}C_f)$ whenever $C_f=\Omega(n)$, the same order as $n_{\text{adam}}$ plain Adam steps. Because $n_{\text{pop}}=4$ is constant, EvE's cost stays within a constant factor of a single Adam optimizer at any dimension, which is what lets Section~\ref{sec:experiments} reach $n=10^6$ without a larger budget.

\paragraph{Memory cost.} EvE trades a constant, dimension-independent memory overhead for its population and persistent momentum: $13\times$ the live parameter count in total storage against a plain Adam optimizer's $4\times$, a $\sim3$--$4\times$ multiplier that is fixed and does not grow with problem size. Appendix~\ref{app:algo} gives the full derivation and a measured correction to this theoretical count.

\paragraph{Reproducibility.} Every stochastic draw comes from one of two persistent generators seeded once from a single integer $s$: a device generator used only at initialization (Section~\ref{sec:init}), and a CPU generator drawn from continuously for every donor selection, jitter, and mutation seed thereafter. Neither is reseeded mid-run, so a run is exactly reproducible from $s$.

\vspace{-2mm}
\section{Experiments}
\label{sec:experiments}
\vspace{-2mm}

\subsection{Experimental Setup}
\label{sec:setup-exp}

Every comparison in this section holds EvE and Adam to the same evaluation-cost budget $B$, so any gap in final quality reflects the optimizer, not a compute advantage. We budget by function evaluations, the native currency of evolutionary search, but a gradient is not free: Adam takes one every evaluation while EvE's fallback only takes one when a candidate needs rescuing (Section~\ref{sec:adam-fallback}), so we convert every backward pass into evaluation-equivalent units from its own measured wall time, never an assumed multiplier (Appendix~\ref{app:eval-cost} gives the exact rule). A run stops the instant its cumulative cost reaches $B$; we use $B=2{,}000$ on the synthetic benchmarks, $3{,}000$ on MLP/MNIST, and $300$ on LoRA/Alpaca (raised $3\times$ to $900$ on GSM8K, a deliberate and stated deviation; Appendix~\ref{app:gsm8k}). Where a method is reported at its own best learning rate, that rate is chosen by the mean held-out test metric over seeds, the same rule for every method; Appendix~\ref{app:lr-selection} reports what changes if it is chosen on validation data instead.

\vspace{-2mm}
\subsection{Synthetic Scalable Benchmarks}
\label{sec:benchmarks}

We evaluate on seven classical scalable benchmark problems (Ackley, Griewank, Rastrigin, Rosenbrock, Schwefel, Sphere, Zakharov), each implemented from its standard closed-form definition so that both EvE and Adam optimize the identical objective, swept over $n\in\{2,10,20,50,100,200,500,1{,}000,10^5,10^6\}$ at 21 independent seeds each, under the shared budget $B=2{,}000$ of Section~\ref{sec:setup-exp} (Sphere is recentered to keep its optimum interior to the search box rather than on its boundary; Appendix~\ref{app:benchmarks} explains why). Table~\ref{tab:scalable-benchmarks} gives the final objective value (lower is better) at every (problem, $n$) cell, with a one-sided Wilcoxon signed-rank test (paired by seed, Holm-Bonferroni corrected across all cells at family-wise $\alpha=0.05$) marking whichever side, if either, is significantly better; Table~\ref{tab:scalable-summary} (Appendix~\ref{app:benchmarks}) counts the wins, ties and losses per problem.

\begin{table}[t]
\centering
\caption{EvE vs.\ Adam on the scalable synthetic benchmarks, at fixed
evaluation-cost budget (Section~\ref{sec:experiments}). Lower is better.
Min/Med/Max over 21 seeds per (problem, $n$) pair. A shaded, bold median is
the significantly lower (better) side (one-sided Wilcoxon signed-rank,
Holm-Bonferroni corrected, family-wise $\alpha=0.05$); shaded, bold, italic
on both sides marks no significant difference.}
\label{tab:scalable-benchmarks}
\scalebox{0.6}{%
\begin{tabular}{|l|c|c|c|c|c|c|c|c|}
\hline
\multirow{2}{*}{Problem} & \multirow{2}{*}{Visualization} & \multirow{2}{*}{$n$} & \multicolumn{3}{c|}{EvE} & \multicolumn{3}{c|}{Adam} \\
\cline{4-9}
 & & & Min & Med & Max & Min & Med & Max \\
\hline
\multirow{10}{*}{Ackley~\citep{ackley1987connectionist}} & \multirow{10}{*}{\includegraphics[width=2.2cm]{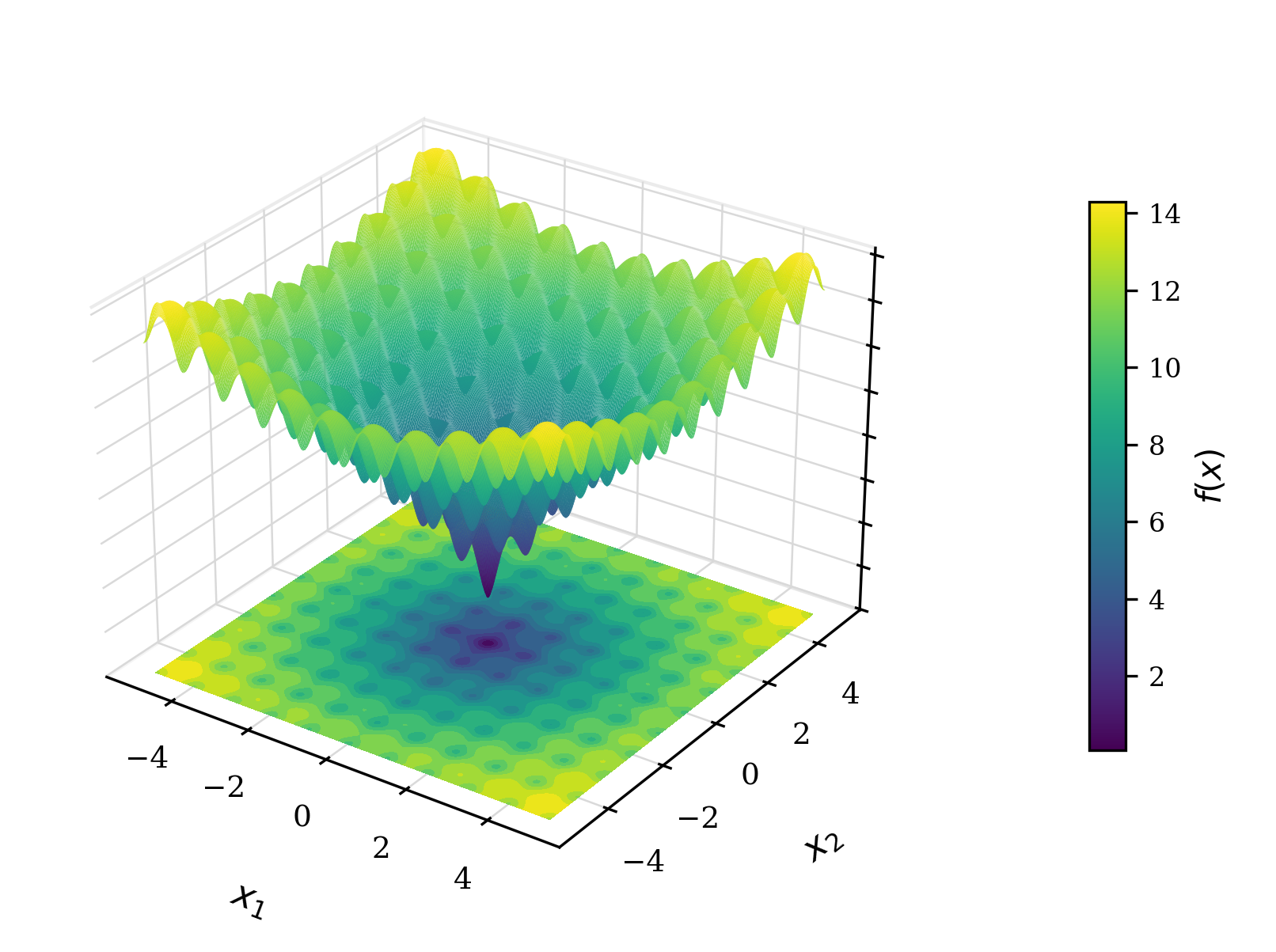}} & 2 & 1.24e-5 & \cellcolor{gray!25}\textbf{3.00e-4} & 2.58e0 & 6.88e0 & 1.88e1 & 1.99e1 \\
\cline{3-9}
 &  & 10 & 2.58e0 & \cellcolor{gray!25}\textbf{1.58e1} & 1.98e1 & 1.86e1 & 1.95e1 & 1.99e1 \\
\cline{3-9}
 &  & 20 & 4.95e0 & \cellcolor{gray!25}\textbf{1.55e1} & 2.01e1 & 1.86e1 & 1.94e1 & 1.98e1 \\
\cline{3-9}
 &  & 50 & 4.82e0 & \cellcolor{gray!25}\textbf{\textit{1.80e1}} & 2.04e1 & 1.91e1 & \cellcolor{gray!25}\textbf{\textit{1.95e1}} & 1.98e1 \\
\cline{3-9}
 &  & 100 & 9.54e0 & \cellcolor{gray!25}\textbf{1.65e1} & 2.03e1 & 1.94e1 & 1.96e1 & 1.97e1 \\
\cline{3-9}
 &  & 200 & 1.21e1 & \cellcolor{gray!25}\textbf{1.69e1} & 2.05e1 & 1.94e1 & 1.96e1 & 1.97e1 \\
\cline{3-9}
 &  & 500 & 6.68e0 & \cellcolor{gray!25}\textbf{\textit{1.80e1}} & 2.04e1 & 1.95e1 & \cellcolor{gray!25}\textbf{\textit{1.96e1}} & 1.96e1 \\
\cline{3-9}
 &  & 1{,}000 & 6.79e0 & \cellcolor{gray!25}\textbf{\textit{1.82e1}} & 2.04e1 & 1.95e1 & \cellcolor{gray!25}\textbf{\textit{1.96e1}} & 1.96e1 \\
\cline{3-9}
 &  & 100{,}000 & 9.72e-4 & \cellcolor{gray!25}\textbf{4.40e0} & 1.03e1 & 1.96e1 & 1.96e1 & 1.96e1 \\
\cline{3-9}
 &  & 1{,}000{,}000 & 5.69e-4 & \cellcolor{gray!25}\textbf{6.68e0} & 1.07e1 & 1.96e1 & 1.96e1 & 1.96e1 \\
\hline
\multirow{10}{*}{Griewank~\citep{griewank1981generalized}} & \multirow{10}{*}{\includegraphics[width=2.2cm]{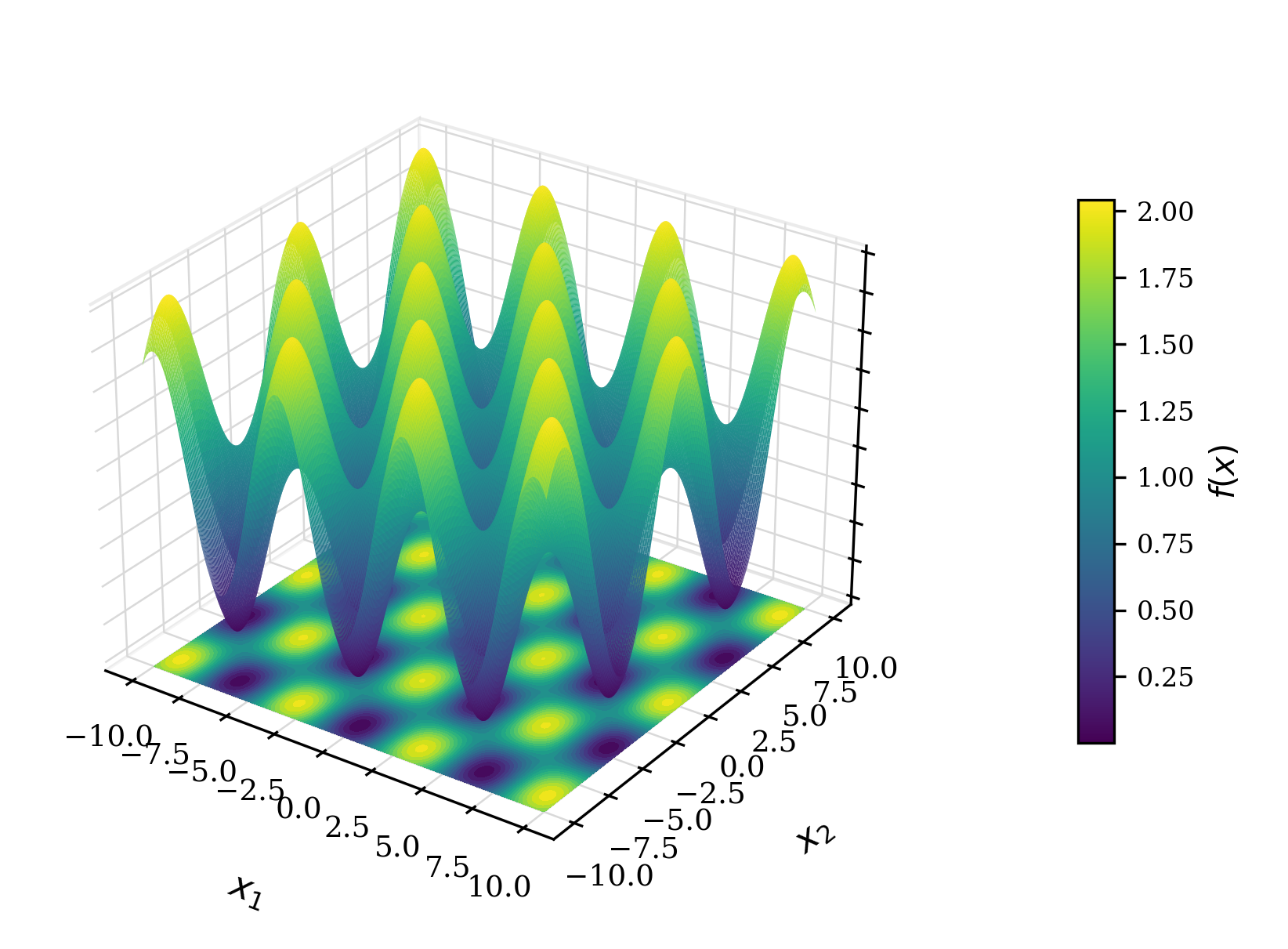}} & 2 & 9.57e-3 & \cellcolor{gray!25}\textbf{4.61e-2} & 2.83e-1 & 5.99e-1 & 3.28e1 & 1.45e2 \\
\cline{3-9}
 &  & 10 & 4.30e-2 & \cellcolor{gray!25}\textbf{8.97e-1} & 1.68e0 & 9.64e1 & 2.01e2 & 3.93e2 \\
\cline{3-9}
 &  & 20 & 1.33e0 & \cellcolor{gray!25}\textbf{3.15e0} & 3.57e1 & 1.98e2 & 3.70e2 & 6.23e2 \\
\cline{3-9}
 &  & 50 & 6.11e0 & \cellcolor{gray!25}\textbf{6.89e1} & 1.83e2 & 6.91e2 & 9.90e2 & 1.39e3 \\
\cline{3-9}
 &  & 100 & 1.63e2 & \cellcolor{gray!25}\textbf{2.36e2} & 7.13e2 & 1.74e3 & 2.16e3 & 2.53e3 \\
\cline{3-9}
 &  & 200 & 4.37e1 & \cellcolor{gray!25}\textbf{5.91e2} & 1.79e3 & 3.60e3 & 4.23e3 & 4.70e3 \\
\cline{3-9}
 &  & 500 & 1.43e3 & \cellcolor{gray!25}\textbf{1.91e3} & 4.94e3 & 9.65e3 & 1.04e4 & 1.16e4 \\
\cline{3-9}
 &  & 1{,}000 & 9.23e2 & \cellcolor{gray!25}\textbf{4.45e3} & 9.87e3 & 1.97e4 & 2.11e4 & 2.29e4 \\
\cline{3-9}
 &  & 100{,}000 & 2.67e4 & \cellcolor{gray!25}\textbf{7.79e4} & 9.15e5 & 1.96e6 & 2.00e6 & 2.05e6 \\
\cline{3-9}
 &  & 1{,}000{,}000 & 2.72e5 & \cellcolor{gray!25}\textbf{3.71e6} & 9.05e6 & 1.57e7 & 1.60e7 & 1.66e7 \\
\hline
\multirow{10}{*}{Rastrigin~\citep{rastrigin1974systems}} & \multirow{10}{*}{\includegraphics[width=2.2cm]{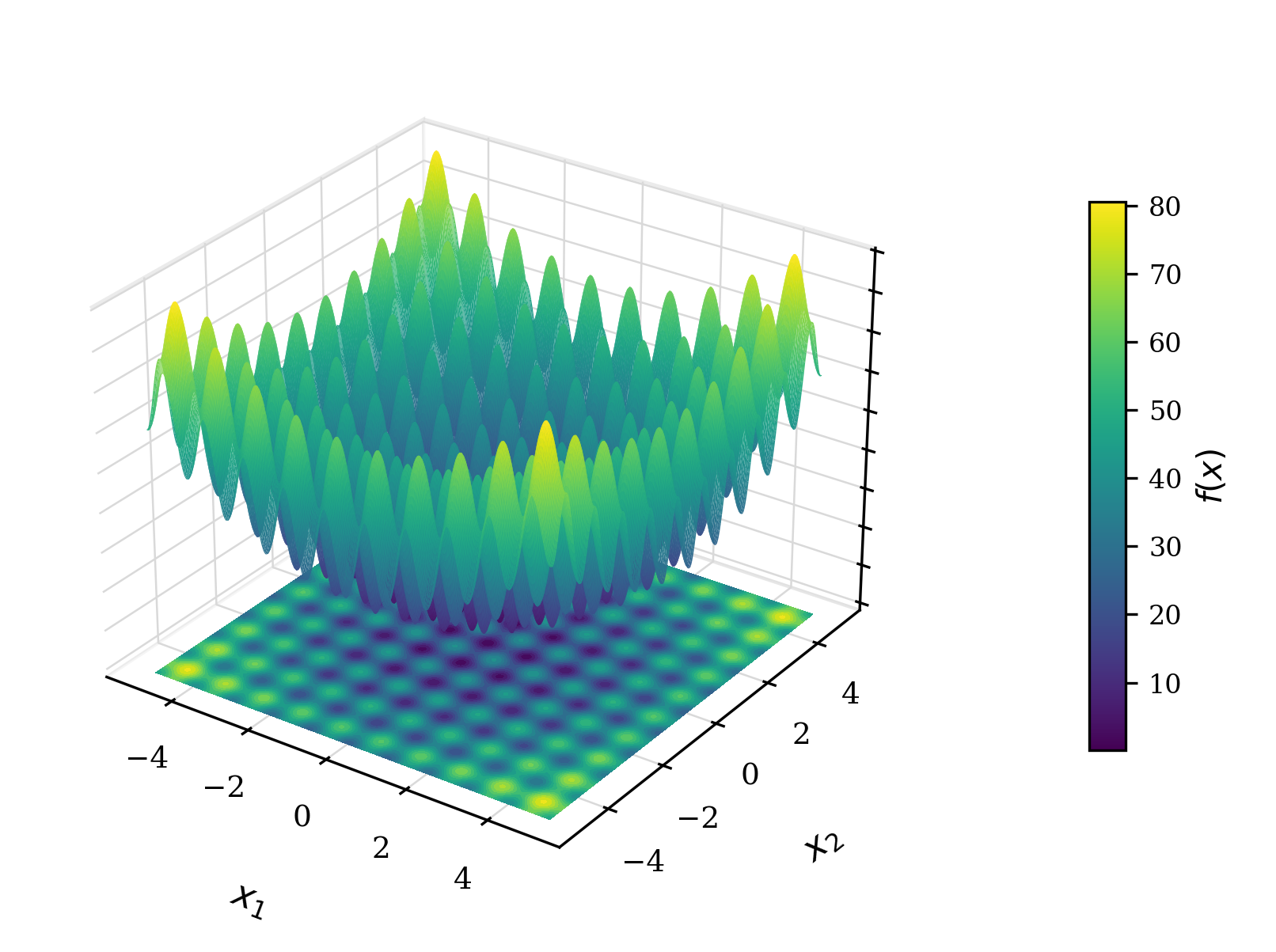}} & 2 & 0 & \cellcolor{gray!25}\textbf{9.95e-1} & 5.44e0 & 0 & 8.95e0 & 4.08e1 \\
\cline{3-9}
 &  & 10 & 1.42e1 & \cellcolor{gray!25}\textbf{3.51e1} & 7.06e1 & 4.18e1 & 7.96e1 & 1.44e2 \\
\cline{3-9}
 &  & 20 & 2.95e1 & \cellcolor{gray!25}\textbf{\textit{9.95e1}} & 2.20e2 & 8.86e1 & \cellcolor{gray!25}\textbf{\textit{1.50e2}} & 2.39e2 \\
\cline{3-9}
 &  & 50 & 2.39e1 & \cellcolor{gray!25}\textbf{1.82e2} & 5.06e2 & 2.82e2 & 3.88e2 & 5.25e2 \\
\cline{3-9}
 &  & 100 & 1.46e2 & \cellcolor{gray!25}\textbf{6.08e2} & 1.11e3 & 7.00e2 & 8.58e2 & 9.82e2 \\
\cline{3-9}
 &  & 200 & 1.51e2 & \cellcolor{gray!25}\textbf{1.18e3} & 1.93e3 & 1.44e3 & 1.67e3 & 1.84e3 \\
\cline{3-9}
 &  & 500 & 9.59e2 & \cellcolor{gray!25}\textbf{2.44e3} & 5.84e3 & 3.78e3 & 4.14e3 & 4.56e3 \\
\cline{3-9}
 &  & 1{,}000 & 1.17e3 & \cellcolor{gray!25}\textbf{\textit{5.80e3}} & 1.31e4 & 7.72e3 & \cellcolor{gray!25}\textbf{\textit{8.34e3}} & 9.09e3 \\
\cline{3-9}
 &  & 100{,}000 & 1.50e5 & \cellcolor{gray!25}\textbf{5.06e5} & 1.15e6 & 8.31e5 & 8.37e5 & 8.41e5 \\
\cline{3-9}
 &  & 1{,}000{,}000 & 1.15e6 & \cellcolor{gray!25}\textbf{\textit{6.09e6}} & 1.26e7 & 8.35e6 & \cellcolor{gray!25}\textbf{\textit{8.37e6}} & 8.40e6 \\
\hline
\multirow{10}{*}{Rosenbrock~\citep{rosenbrock1960automatic}} & \multirow{10}{*}{\includegraphics[width=2.2cm]{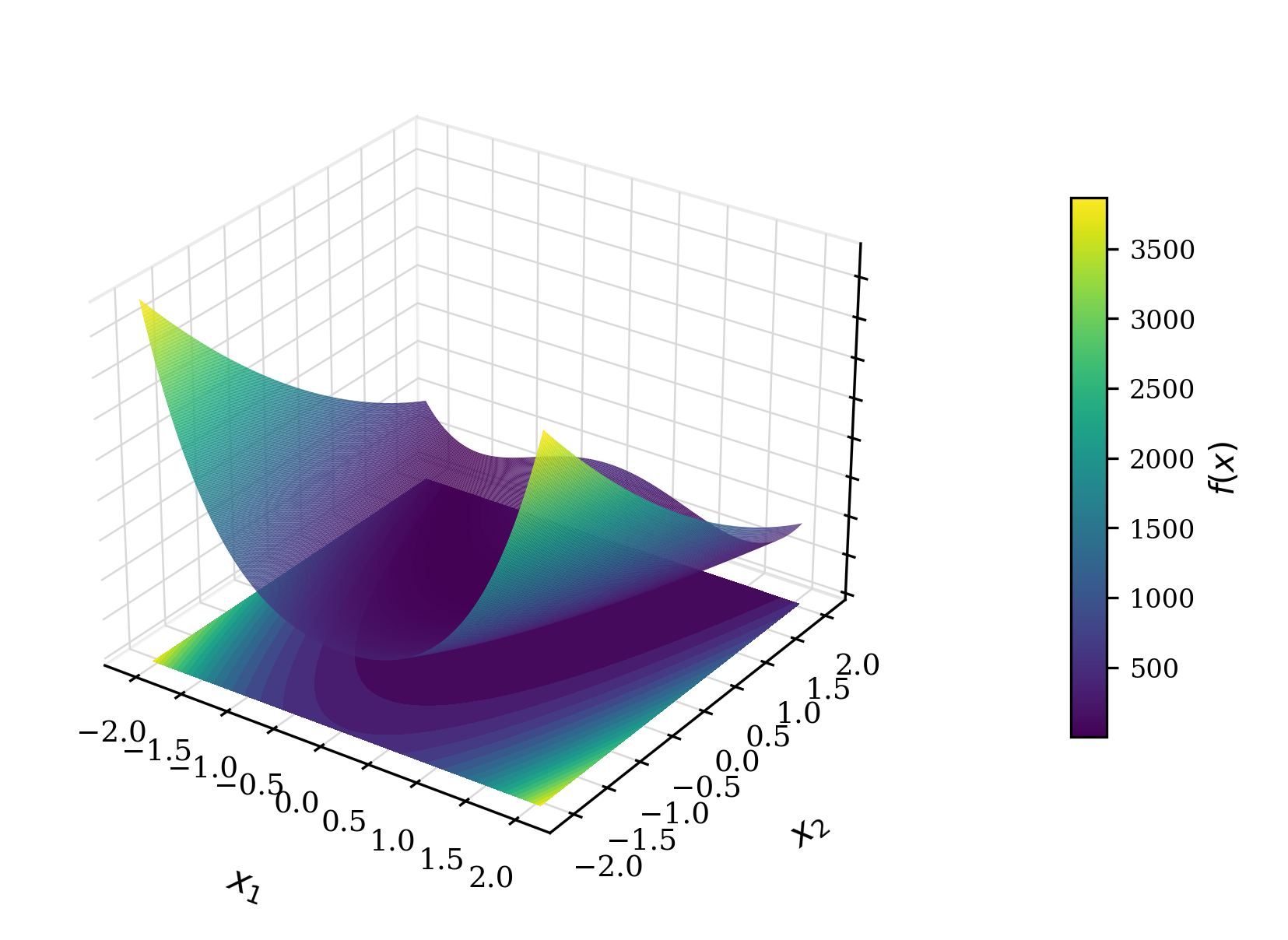}} & 2 & 5.68e-14 & \cellcolor{gray!25}\textbf{\textit{3.17e-2}} & 2.41e-1 & 1.98e-7 & \cellcolor{gray!25}\textbf{\textit{8.16e-3}} & 7.75e-1 \\
\cline{3-9}
 &  & 10 & 8.47e-3 & \cellcolor{gray!25}\textbf{\textit{3.25e0}} & 8.84e0 & 3.90e-2 & \cellcolor{gray!25}\textbf{\textit{5.00e0}} & 9.41e0 \\
\cline{3-9}
 &  & 20 & 3.63e-1 & \cellcolor{gray!25}\textbf{\textit{1.29e1}} & 1.94e1 & 7.14e0 & \cellcolor{gray!25}\textbf{\textit{1.48e1}} & 1.93e1 \\
\cline{3-9}
 &  & 50 & 2.73e1 & \cellcolor{gray!25}\textbf{\textit{4.49e1}} & 1.01e2 & 3.71e1 & \cellcolor{gray!25}\textbf{\textit{4.44e1}} & 4.90e1 \\
\cline{3-9}
 &  & 100 & 7.51e1 & \cellcolor{gray!25}\textbf{\textit{9.30e1}} & 1.45e2 & 8.62e1 & \cellcolor{gray!25}\textbf{\textit{9.39e1}} & 9.85e1 \\
\cline{3-9}
 &  & 200 & 1.72e2 & \cellcolor{gray!25}\textbf{\textit{1.94e2}} & 2.48e2 & 1.86e2 & \cellcolor{gray!25}\textbf{\textit{1.93e2}} & 1.97e2 \\
\cline{3-9}
 &  & 500 & 4.71e2 & \cellcolor{gray!25}\textbf{\textit{4.93e2}} & 6.43e2 & 4.82e2 & \cellcolor{gray!25}\textbf{\textit{4.91e2}} & 5.41e2 \\
\cline{3-9}
 &  & 1{,}000 & 9.72e2 & \cellcolor{gray!25}\textbf{\textit{9.97e2}} & 1.19e3 & 9.77e2 & \cellcolor{gray!25}\textbf{\textit{9.88e2}} & 1.04e3 \\
\cline{3-9}
 &  & 100{,}000 & 9.90e4 & \cellcolor{gray!25}\textbf{\textit{1.02e5}} & 1.24e5 & 1.01e5 & \cellcolor{gray!25}\textbf{\textit{1.02e5}} & 1.03e5 \\
\cline{3-9}
 &  & 1{,}000{,}000 & 9.90e5 & \cellcolor{gray!25}\textbf{\textit{1.03e6}} & 1.31e6 & 1.02e6 & \cellcolor{gray!25}\textbf{\textit{1.02e6}} & 1.03e6 \\
\hline
\multirow{10}{*}{Schwefel~\citep{schwefel1981numerical}} & \multirow{10}{*}{\includegraphics[width=2.2cm]{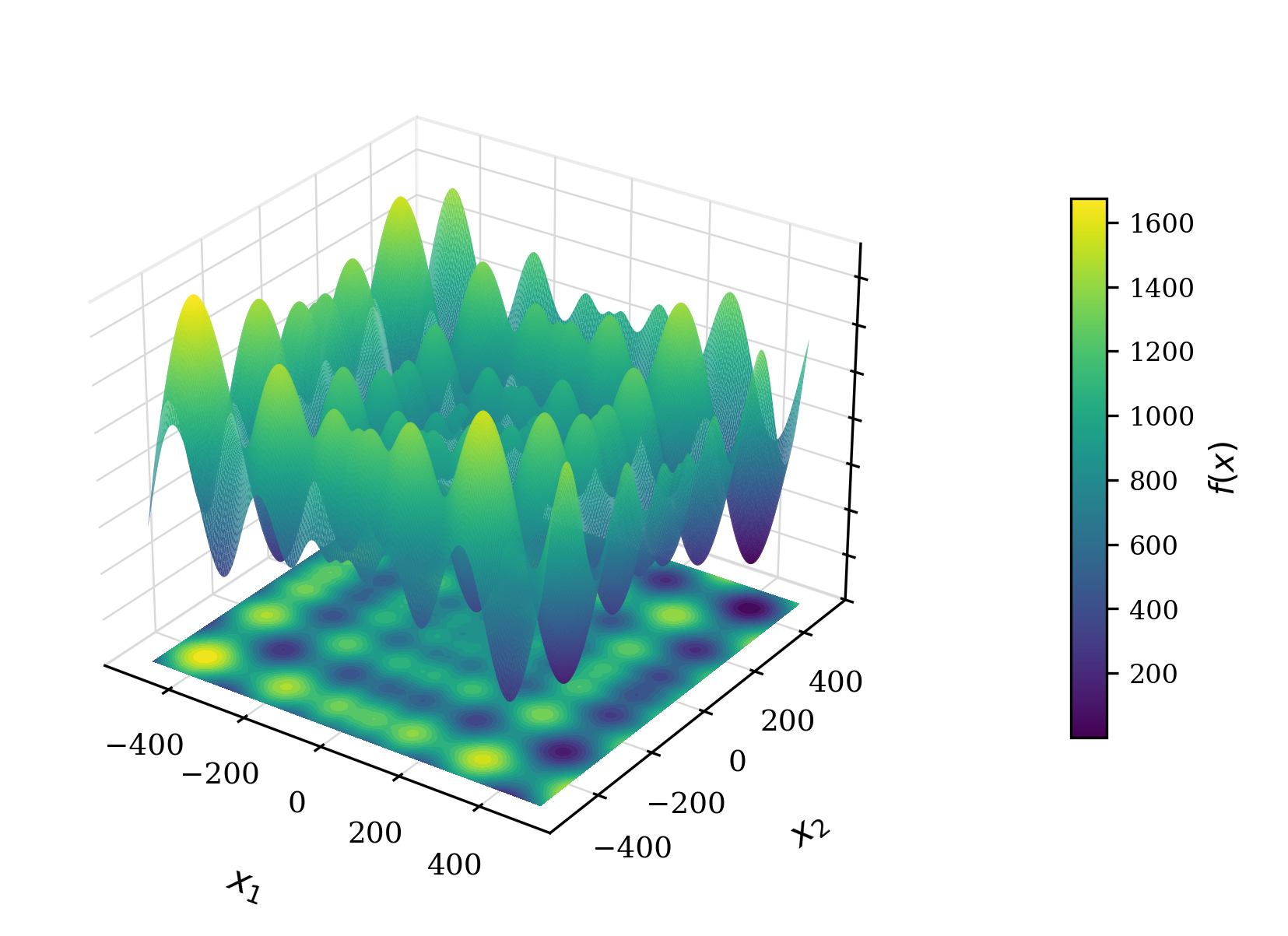}} & 2 & 0 & \cellcolor{gray!25}\textbf{1.18e2} & 4.01e2 & 1.22e2 & 5.16e2 & 8.10e2 \\
\cline{3-9}
 &  & 10 & 8.90e2 & \cellcolor{gray!25}\textbf{\textit{1.46e3}} & 3.53e3 & 1.10e3 & \cellcolor{gray!25}\textbf{\textit{2.09e3}} & 3.05e3 \\
\cline{3-9}
 &  & 20 & 2.52e3 & \cellcolor{gray!25}\textbf{\textit{4.37e3}} & 7.62e3 & 3.34e3 & \cellcolor{gray!25}\textbf{\textit{4.42e3}} & 5.62e3 \\
\cline{3-9}
 &  & 50 & 8.81e3 & 1.47e4 & 1.93e4 & 8.91e3 & \cellcolor{gray!25}\textbf{1.06e4} & 1.29e4 \\
\cline{3-9}
 &  & 100 & 2.11e4 & 3.80e4 & 3.99e4 & 1.84e4 & \cellcolor{gray!25}\textbf{2.09e4} & 2.38e4 \\
\cline{3-9}
 &  & 200 & 5.71e4 & 7.50e4 & 8.18e4 & 3.77e4 & \cellcolor{gray!25}\textbf{4.23e4} & 4.45e4 \\
\cline{3-9}
 &  & 500 & 1.58e5 & 1.71e5 & 2.02e5 & 9.94e4 & \cellcolor{gray!25}\textbf{1.04e5} & 1.11e5 \\
\cline{3-9}
 &  & 1{,}000 & 3.21e5 & 3.55e5 & 4.12e5 & 2.00e5 & \cellcolor{gray!25}\textbf{2.07e5} & 2.16e5 \\
\cline{3-9}
 &  & 100{,}000 & 3.18e7 & 3.47e7 & 3.81e7 & 2.02e7 & \cellcolor{gray!25}\textbf{2.04e7} & 2.07e7 \\
\cline{3-9}
 &  & 1{,}000{,}000 & 3.22e8 & 3.68e8 & 4.17e8 & 2.01e8 & \cellcolor{gray!25}\textbf{2.02e8} & 4.19e8 \\
\hline
\multirow{10}{*}{Sphere~\citep{dejong1975analysis}} & \multirow{10}{*}{\includegraphics[width=2.2cm]{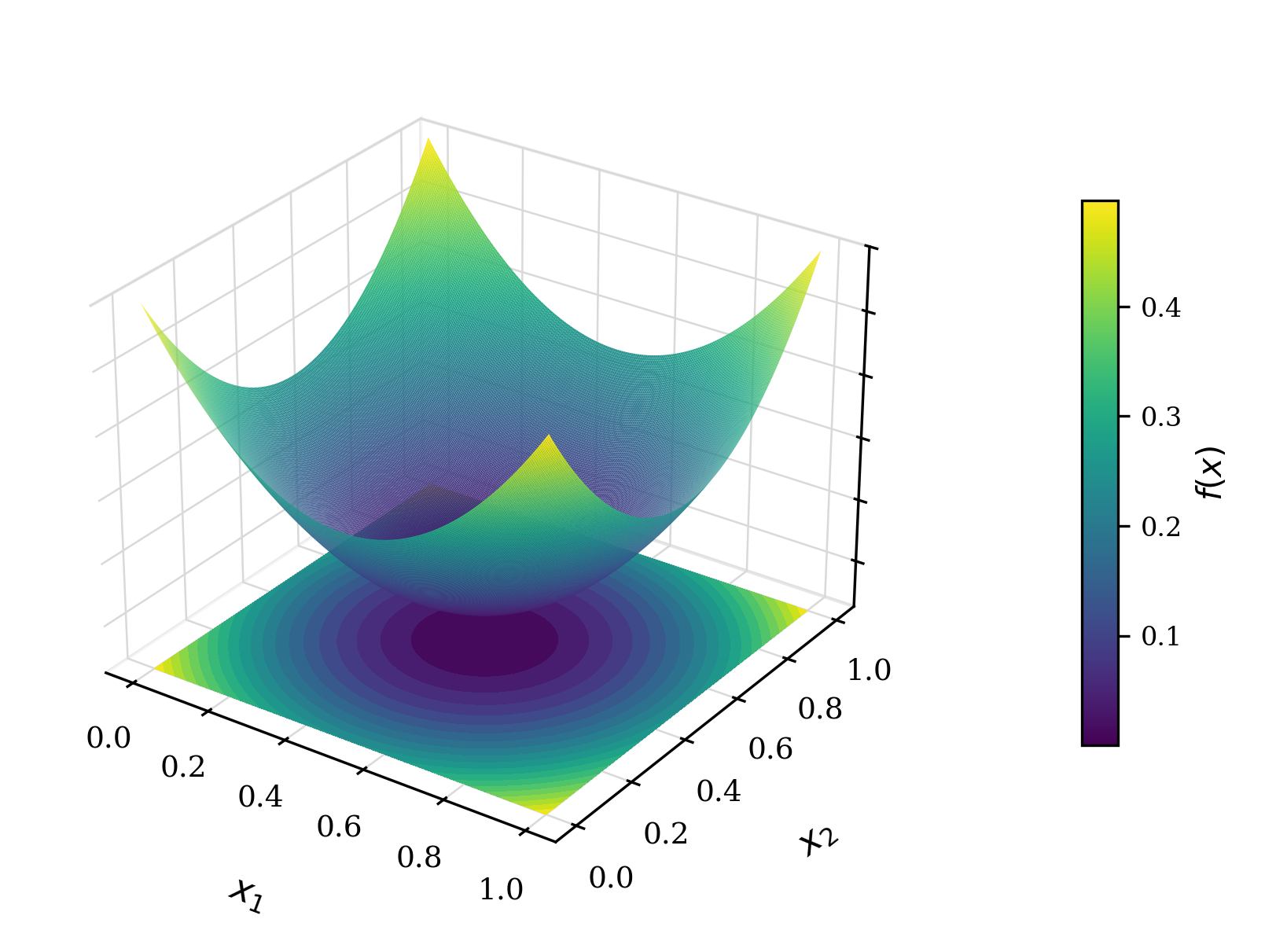}} & 2 & 4.71e-14 & 4.44e-11 & 2.73e-10 & 0 & \cellcolor{gray!25}\textbf{0} & 0 \\
\cline{3-9}
 &  & 10 & 3.03e-9 & 4.62e-9 & 5.78e-9 & 0 & \cellcolor{gray!25}\textbf{0} & 4.44e-15 \\
\cline{3-9}
 &  & 20 & 8.73e-9 & 1.10e-8 & 1.26e-8 & 0 & \cellcolor{gray!25}\textbf{1.78e-15} & 7.11e-15 \\
\cline{3-9}
 &  & 50 & 2.68e-8 & 2.98e-8 & 3.61e-8 & 1.78e-15 & \cellcolor{gray!25}\textbf{1.07e-14} & 1.87e-14 \\
\cline{3-9}
 &  & 100 & 5.90e-8 & 6.23e-8 & 1.18e-7 & 1.60e-14 & \cellcolor{gray!25}\textbf{2.49e-14} & 4.80e-14 \\
\cline{3-9}
 &  & 200 & 1.22e-7 & 1.27e-7 & 1.87e-7 & 3.73e-14 & \cellcolor{gray!25}\textbf{6.39e-14} & 8.53e-14 \\
\cline{3-9}
 &  & 500 & 3.08e-7 & 3.25e-7 & 5.69e-7 & 1.41e-13 & \cellcolor{gray!25}\textbf{1.65e-13} & 2.21e-13 \\
\cline{3-9}
 &  & 1{,}000 & 6.23e-7 & 6.60e-7 & 1.16e-6 & 2.74e-13 & \cellcolor{gray!25}\textbf{3.55e-13} & 4.40e-13 \\
\cline{3-9}
 &  & 100{,}000 & 6.25e-5 & 6.27e-5 & 6.67e-5 & 3.14e-11 & \cellcolor{gray!25}\textbf{3.24e-11} & 3.39e-11 \\
\cline{3-9}
 &  & 1{,}000{,}000 & 6.25e-4 & 6.26e-4 & 6.35e-4 & 3.06e-10 & \cellcolor{gray!25}\textbf{3.11e-10} & 3.24e-10 \\
\hline
\multirow{10}{*}{Zakharov~\citep{jamil2013literature}} & \multirow{10}{*}{\includegraphics[width=2.2cm]{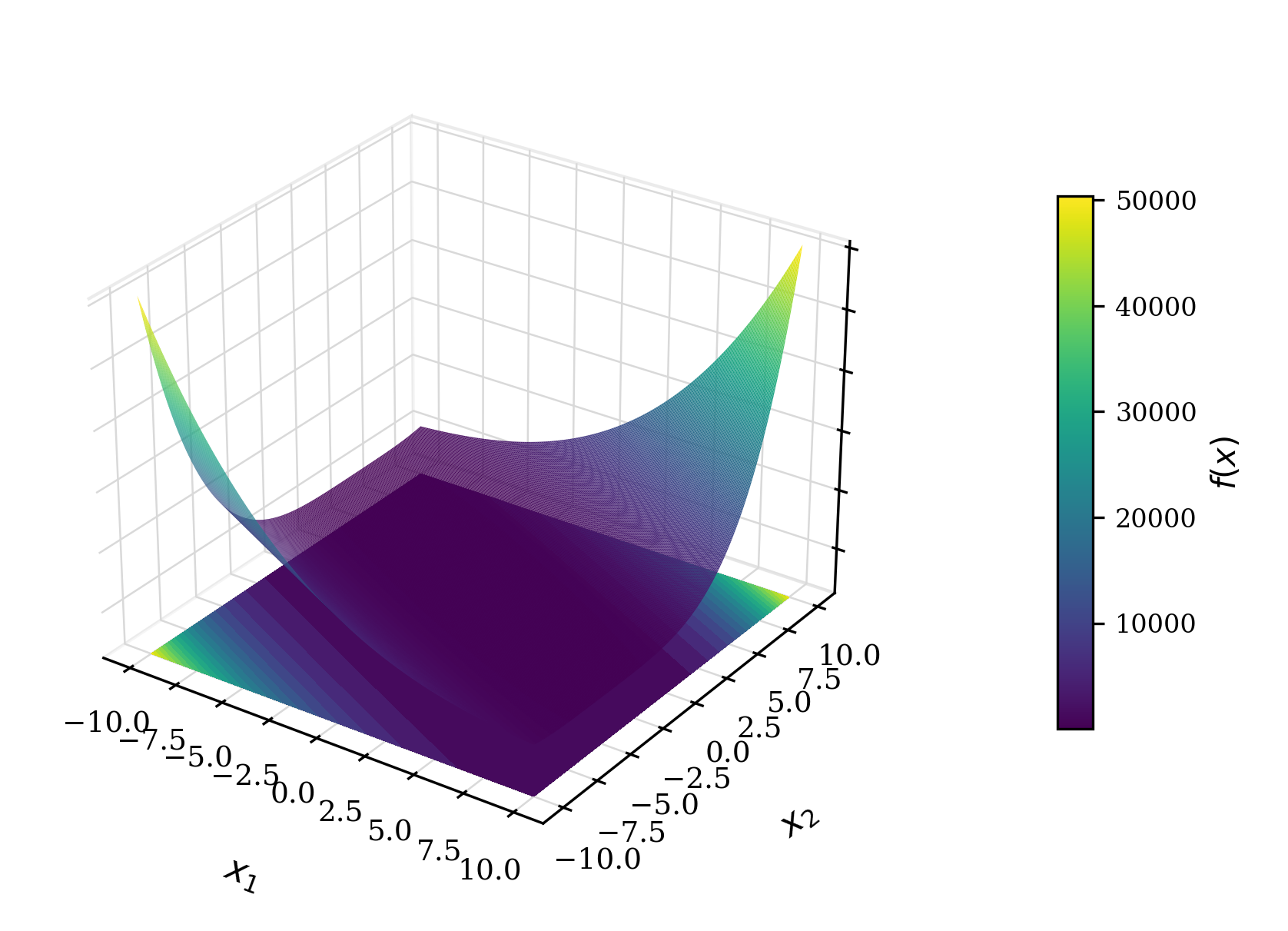}} & 2 & 0 & \cellcolor{gray!25}\textbf{\textit{0}} & 3.61e-5 & 9.60e-34 & \cellcolor{gray!25}\textbf{\textit{1.44e-7}} & 2.98e1 \\
\cline{3-9}
 &  & 10 & 3.02e-6 & \cellcolor{gray!25}\textbf{6.72e0} & 1.94e2 & 2.22e-28 & 2.22e2 & 5.38e2 \\
\cline{3-9}
 &  & 20 & 6.01e1 & \cellcolor{gray!25}\textbf{2.06e2} & 8.54e2 & 3.16e2 & 5.92e2 & 3.65e5 \\
\cline{3-9}
 &  & 50 & 1.53e2 & \cellcolor{gray!25}\textbf{7.83e2} & 3.03e3 & 1.13e3 & 1.77e3 & 7.73e6 \\
\cline{3-9}
 &  & 100 & 6.46e2 & \cellcolor{gray!25}\textbf{\textit{2.93e3}} & 7.37e3 & 2.78e3 & \cellcolor{gray!25}\textbf{\textit{3.56e3}} & 9.41e5 \\
\cline{3-9}
 &  & 200 & 1.43e3 & \cellcolor{gray!25}\textbf{4.57e3} & 8.95e3 & 5.88e3 & 7.15e3 & 8.93e7 \\
\cline{3-9}
 &  & 500 & 3.49e3 & \cellcolor{gray!25}\textbf{1.58e4} & 3.28e4 & 1.73e4 & 2.05e5 & 1.50e9 \\
\cline{3-9}
 &  & 1{,}000 & 3.95e3 & \cellcolor{gray!25}\textbf{2.68e4} & 6.47e4 & 3.42e4 & 1.11e6 & 3.43e10 \\
\cline{3-9}
 &  & 100{,}000 & 1.29e6 & \cellcolor{gray!25}\textbf{2.71e6} & 2.07e27 & 1.71e24 & 1.80e29 & 1.39e32 \\
\cline{3-9}
 &  & 1{,}000{,}000 & 4.30e6 & \cellcolor{gray!25}\textbf{1.40e9} & 3.84e34 & 4.13e32 & 3.96e36 & 2.06e38 \\
\hline
\end{tabular}}
\end{table}

At every dimension swept, EvE has at least as many significantly better cells as Adam across the seven problems. Across the $70$ (problem, $n$) cells, EvE is significantly better in $33$, Adam in $17$, and the remaining $20$ are statistically indistinguishable. The advantage is largest at the extremes: at $n\in\{2,10\}$ EvE is significantly better on four problems and Adam on one (Sphere, where both reach numerical zero), and at $n=10^6$ it still wins three problems to Adam's two, where several of these problems' gradients are least informative or Adam's own iterate diverges numerically. All $17$ of Adam's wins come from two problems, Sphere (every $n$; at $n\le10$ both reach numerical zero, Adam exactly) and Schwefel (from $n=50$), and persist at every larger dimension. Appendix~\ref{app:benchmarks} gives a per-problem discussion and the full parameter table, including the LHS population initialization used only here (Section~\ref{sec:init}).

\vspace{-2mm}
\subsection{Neural Network Training: MLP on MNIST}
\label{sec:mnist}

\begin{wrapfigure}[18]{r}{0.65\linewidth}
\vspace{-3mm}
\centering
\begin{subfigure}[b]{\linewidth}
\includegraphics[width=\linewidth]{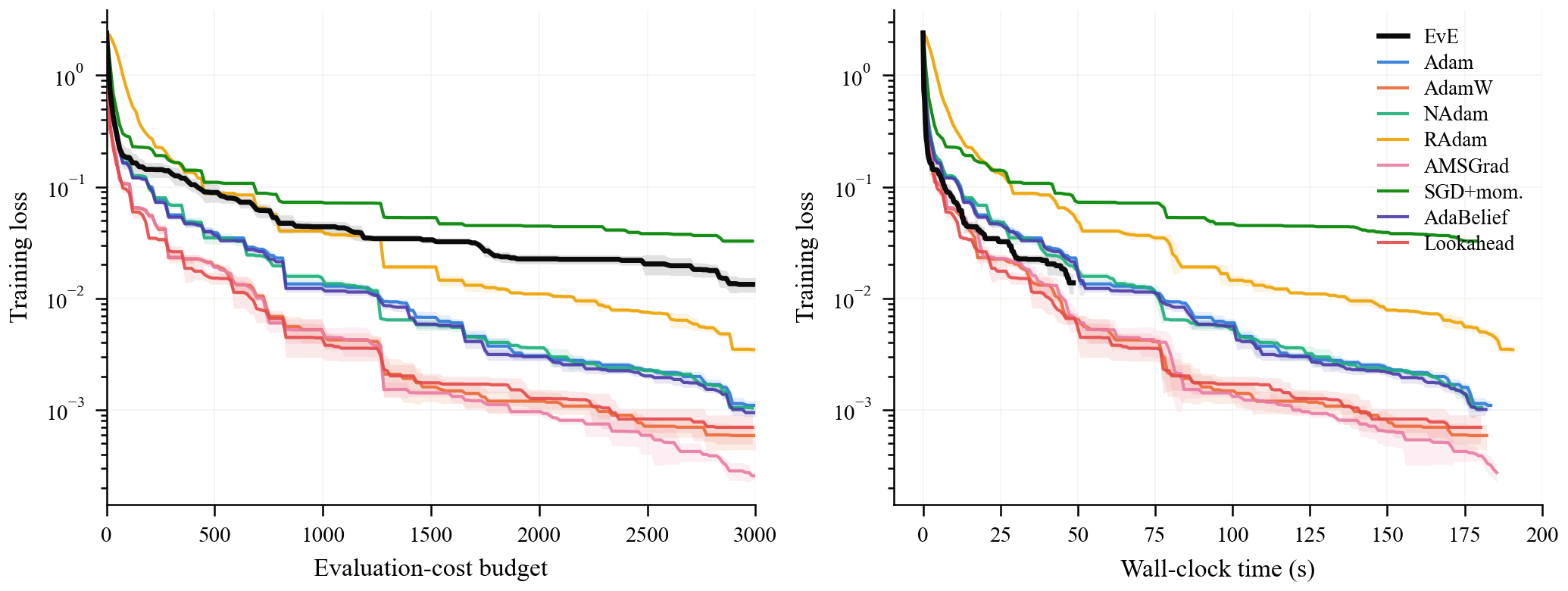}
\caption{Training loss vs.\ evaluation-cost budget (left) and vs.\ wall-clock time (right).}
\end{subfigure}
\caption{EvE vs.\ eight baselines on MLP/MNIST, under the same evaluation-cost budget (Section~\ref{sec:setup-exp}), at each method's own best learning rate over 5 seeds. (a) training runs vs.\ budget and vs.\ wall-clock time; see Figure~\ref{fig:mnist2} for panels (b)--(d).}
\label{fig:mnist}
\end{wrapfigure}

The benchmarks above are hand-built functions with a known closed form; here we ask the same question of a real network learning from real data. We train a two-hidden-layer ($1{,}024$-unit) MLP on MNIST digit classification \citep{lecun1998mnist} ($1.86$M weights, no convolutions) against eight widely used baselines: Adam, AdamW, NAdam, RAdam, AMSGrad, SGD with momentum, AdaBelief, and Lookahead-Adam, each swept over the same learning-rate grid and seeds, reported at its own best learning rate, and stopped at a shared evaluation-cost budget of $B=3{,}000$ (Section~\ref{sec:setup-exp}, Appendix~\ref{app:mnist}).

Figure~\ref{fig:mnist} summarizes the result. EvE finishes that fixed charged work roughly $3.7$--$3.9\times$ faster in wall-clock time ($49$s versus $182$--$192$s for the baselines), at a cost of about one accuracy point ($97.0\%$ versus the Adam family's $98.0$--$98.2\%$), the trade-off predicted by Section~\ref{sec:theory}'s cost analysis, since most of EvE's evaluations go to population upkeep rather than gradient steps. It is also the trade-off that matters for a hyperparameter search or early-pruning loop, where a cheap, fast signal beats a fully polished final model. Appendix~\ref{app:mnist} gives the full architecture, data split, hyperparameter table, and a panel-by-panel discussion of Figure~\ref{fig:mnist}.

\begin{figure}[b]
\ContinuedFloat
\centering
\begin{subfigure}[b]{0.32\linewidth}
\includegraphics[width=\linewidth]{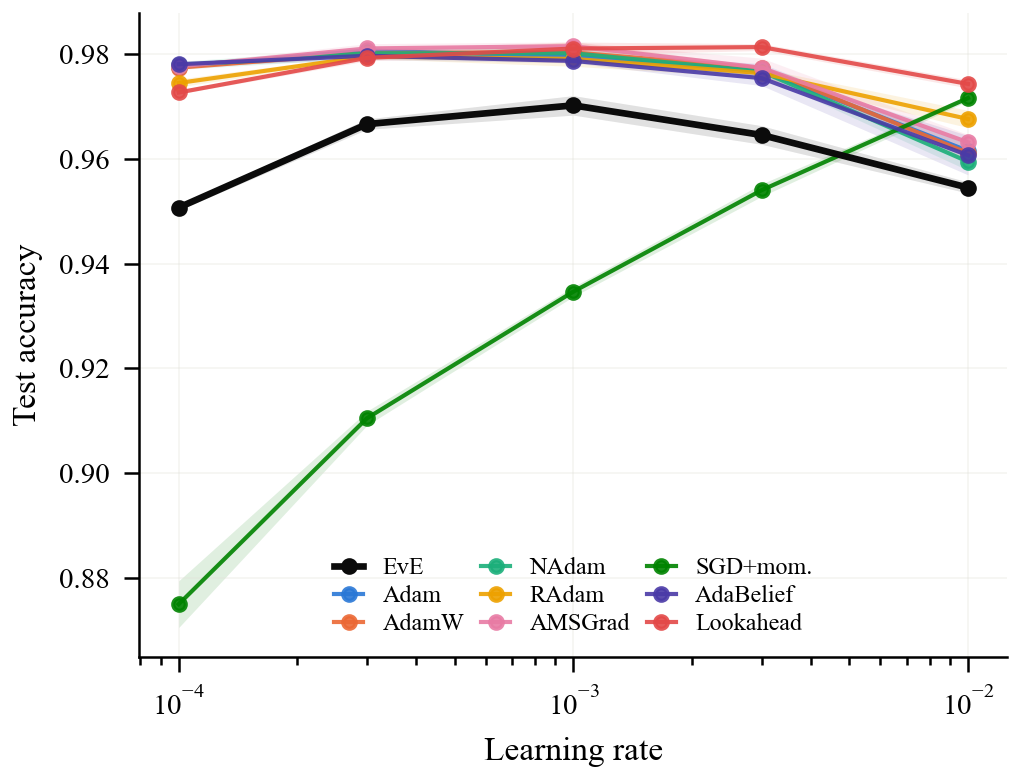}
\caption{Test accuracy vs.\ learning rate.}
\end{subfigure}\hfill
\begin{subfigure}[b]{0.32\linewidth}
\includegraphics[width=\linewidth]{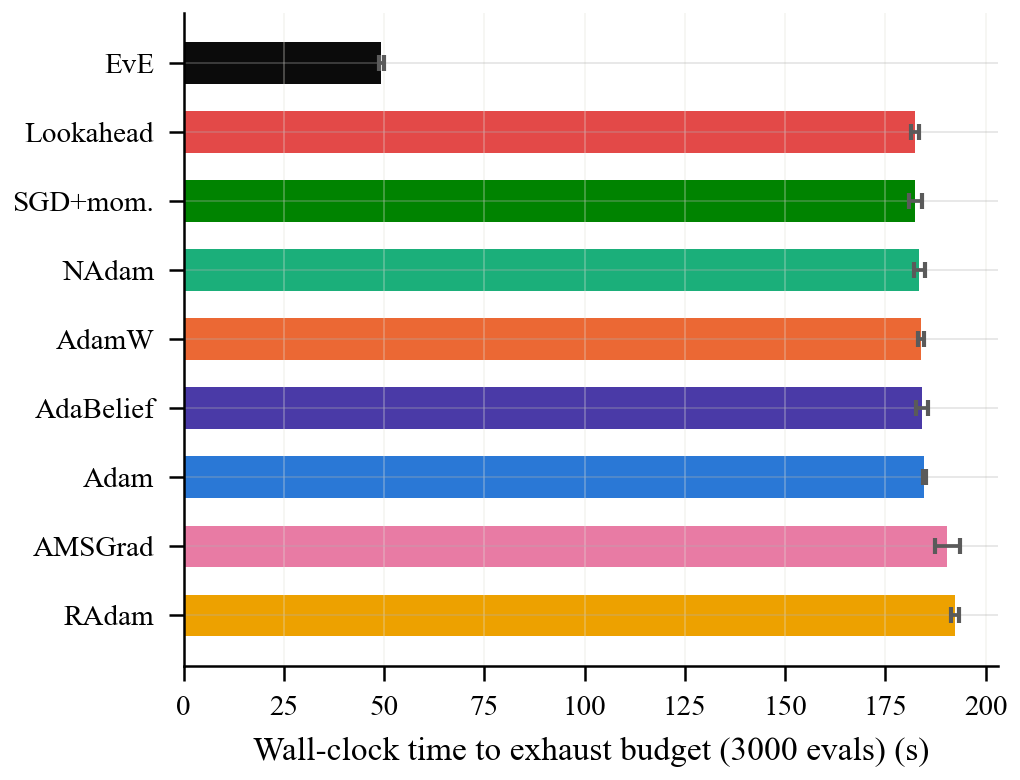}
\caption{Wall-clock time to exhaust the shared budget.}
\end{subfigure}\hfill
\begin{subfigure}[b]{0.32\linewidth}
\includegraphics[width=\linewidth]{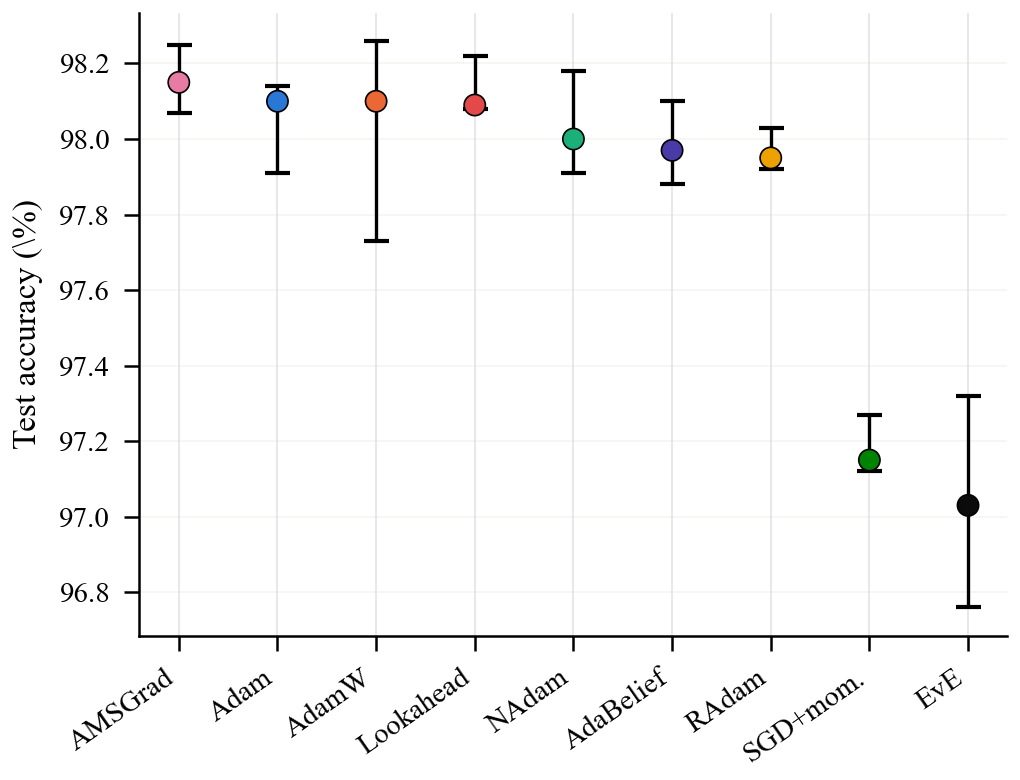}
\caption{Best-LR test accuracy (median, min-max over 5 seeds).}
\end{subfigure}
\caption[]{(cont.) (b) learning-rate sensitivity. (c) wall-clock cost of the shared budget, sorted. (d) final test accuracy reached.}
\label{fig:mnist2}
\end{figure}

\vspace{-2mm}
\subsection{Hyperparameter and Architecture Search: UCI Adult Income}
\label{sec:hpo}

Sections~\ref{sec:benchmarks}--\ref{sec:mnist} compare EvE and Adam as end-to-end trainers under a shared budget; here we instead ask whether EvE is useful one level up, as the inner-loop optimizer inside a search procedure that trains many candidate configurations. We run successive halving (SHA) \citep{jamieson2016non}, promoting only the top third of configurations at each of five geometrically increasing evaluation-cost budgets, on UCI Adult Income \citep{adult1996} (binary income classification, $32{,}561$ train / $16{,}281$ test rows), a standard, non-vision AutoML/HPO benchmark, with a small MLP. Each of the 9 methods runs its own SHA sweep, resumed (not restarted) across rungs, so a promoted configuration's wall-clock time is the honest sum of the budget it actually used. We test two search spaces: (i) learning rate and weight decay, the model architecture fixed; (ii) hidden width and depth, an architecture search, learning rate swept alongside but weight decay fixed. Appendix~\ref{app:hpo} gives the full SHA configuration, search spaces, and per-method tables.

EvE finishes the whole hyperparameter sweep in $21.8\pm0.3$s against $73.6$--$76.8$s for the baselines ($3.4$--$3.5\times$ faster), and the architecture sweep in $21.5\pm0.9$s against $66.9$--$75.8$s ($3.1$--$3.5\times$ faster), reaching a best validation accuracy within about half a point of the baseline cluster in both cases, and outright beating one baseline (sgd\_momentum) on accuracy in the architecture sweep (Figure~\ref{fig:hpo}; Appendix~\ref{app:hpo} gives the exact per-baseline gaps). 
\begin{figure}[hbt]
\centering
\begin{subfigure}[b]{0.45\linewidth}
\includegraphics[width=\linewidth]{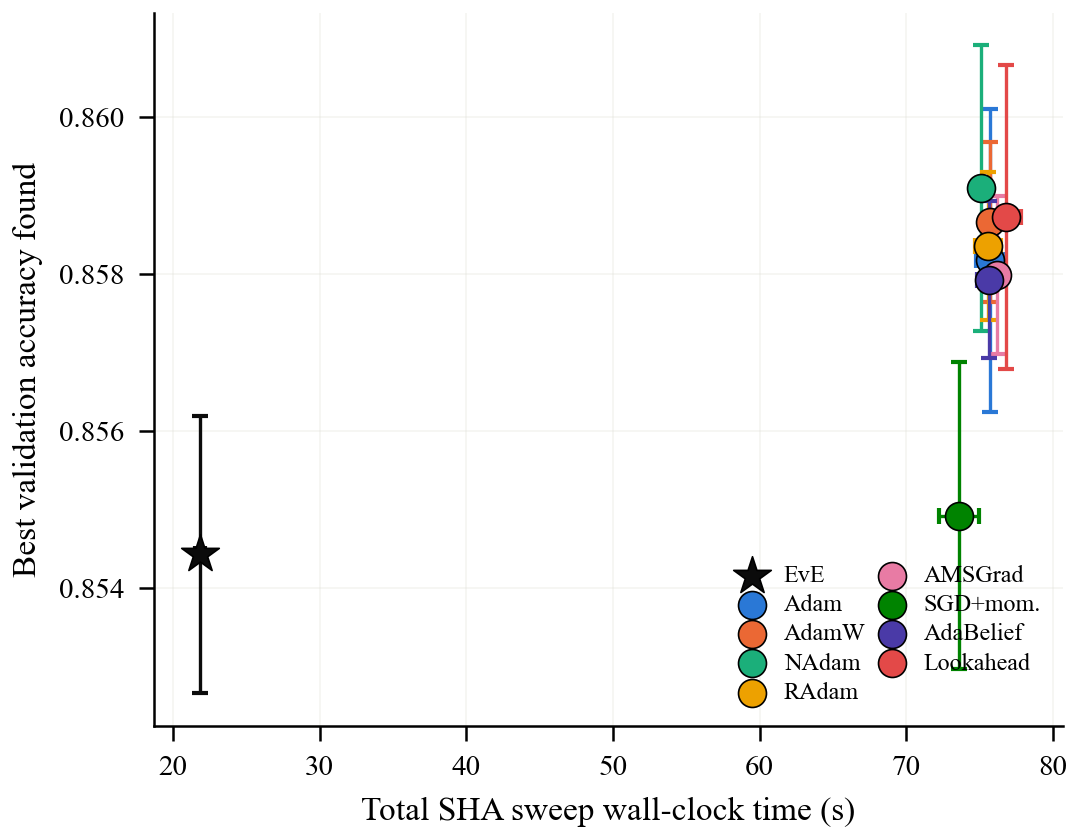}
\caption{Learning rate + weight decay search.}
\end{subfigure}\hfill
\begin{subfigure}[b]{0.45\linewidth}
\includegraphics[width=\linewidth]{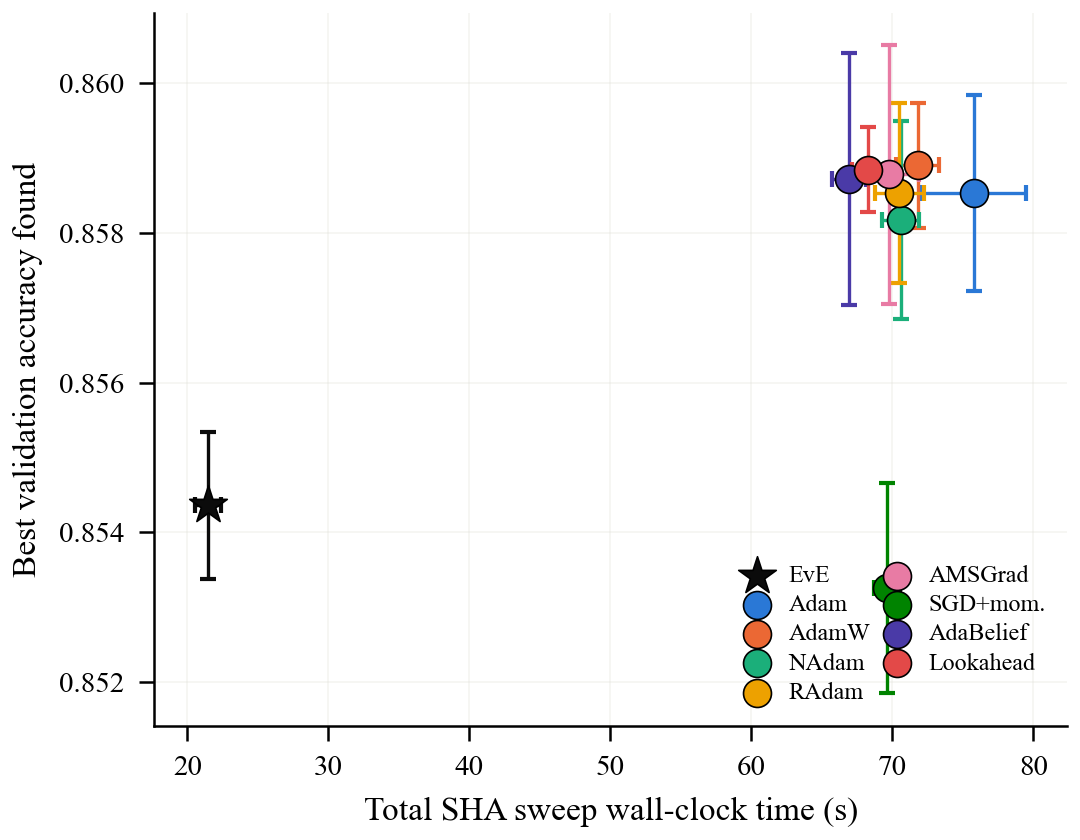}
\caption{Architecture (width + depth) search.}
\end{subfigure}
\caption{SHA sweep wall-clock time vs.\ best validation accuracy found, UCI Adult, mean $\pm$ std over 5 seeds. EvE reaches comparable accuracy in roughly a third of the wall-clock time of every baseline in both search spaces.}
\label{fig:hpo}
\end{figure}
EvE was faster than every one of the eight baselines on all $5$ seeds in both sweeps, a fully consistent trend we report at face value rather than dressing up with a significance test that a sample size this small cannot meaningfully support (Appendix~\ref{app:hpo}). EvE's ranking of the $48$ configurations at the first rung also agrees with Adam's about as well as Adam's agrees with itself across training seeds (Kendall's $\tau$ of $0.66$--$0.69$ against $0.67$--$0.69$; Appendix~\ref{app:rank-agreement}).

\vspace{-2mm}
\subsection{Real-World Fine-Tuning: LoRA on Qwen2.5-1.5B-Instruct}
\label{sec:lora}
\vspace{-1mm}

The two evaluations above are still synthetic in one sense: neither involves a pretrained model or a real instruction-tuning workload. Here we fine-tune Qwen2.5-1.5B-Instruct \citep{qwen2024technical} with LoRA adapters \citep{hu2022lora} (rank $32$, on the query/value/output projections and the MLP, matching torchtune's own reference recipe for this model) on the Alpaca instruction dataset \citep{taori2023alpaca}, in bfloat16, against the same eight baselines as Section~\ref{sec:mnist}, swept over the same learning-rate grid and seeds, reported at each method's best learning rate, and stopped at a shared evaluation-cost budget (Section~\ref{sec:setup-exp}); the only protocol change from Section~\ref{sec:mnist} is the budget size, scaled down for this far more expensive per-evaluation cost (Appendix~\ref{app:lora}).

Figure~\ref{fig:lora} summarizes the result. EvE finishes the shared budget in $35.0\pm0.6$s, against $58.8$--$70.7$s for the baselines at their own best learning rate ($1.7$--$2.0\times$ faster), at a cost of roughly $11\%$ higher test loss ($1.240\pm0.046$ versus the Adam family's $1.115$--$1.124$), the same trade-off found on MNIST (Section~\ref{sec:mnist}), now on a real pretrained model and instruction-tuning task, and it holds at a scale ($1.5$B parameters) two orders of magnitude larger than MNIST's MLP. Appendix~\ref{app:lora} gives the full architecture, data split, hyperparameter table, and a panel-by-panel discussion of Figure~\ref{fig:lora}.

\vspace{-2mm}
\section{Discussion}
\vspace{-2mm}

\subsection{When EvE Helps and When Adam Does}
\label{sec:when-helps}
\vspace{-2mm}

\begin{figure}[t]
\centering
\scalebox{1}{%
\begin{minipage}{\linewidth}
\begin{subfigure}[b]{0.48\linewidth}
\includegraphics[width=\linewidth]{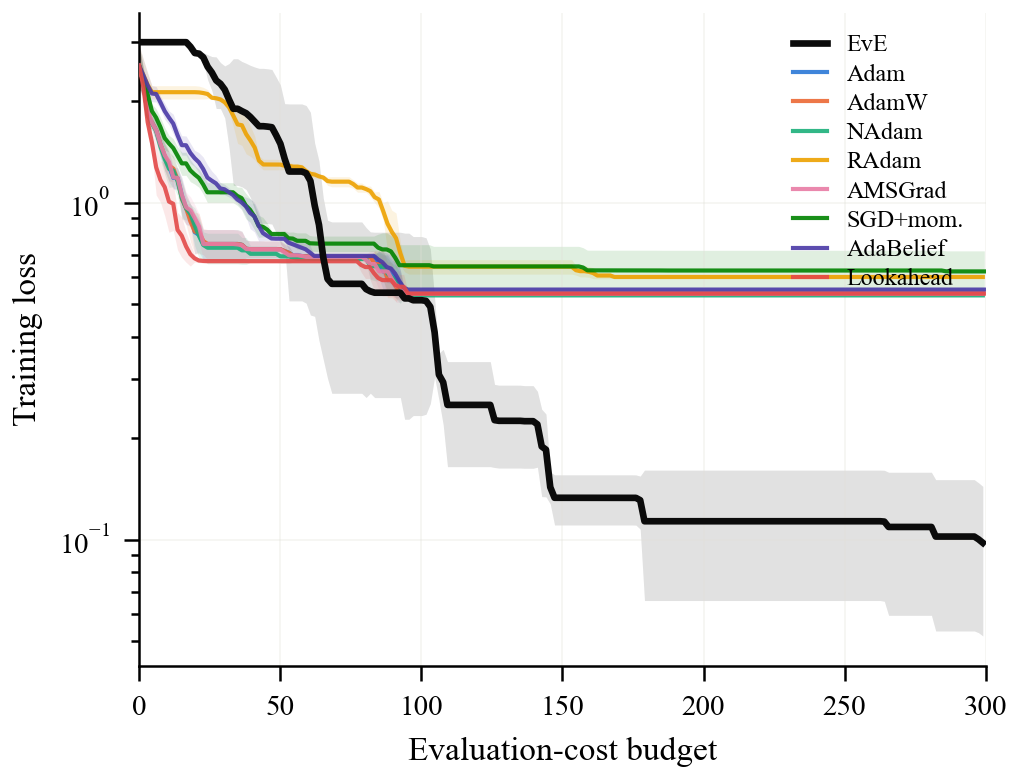}
\caption{Training loss vs.\ evaluation-cost budget.}
\end{subfigure}\hfill
\begin{subfigure}[b]{0.48\linewidth}
\includegraphics[width=\linewidth]{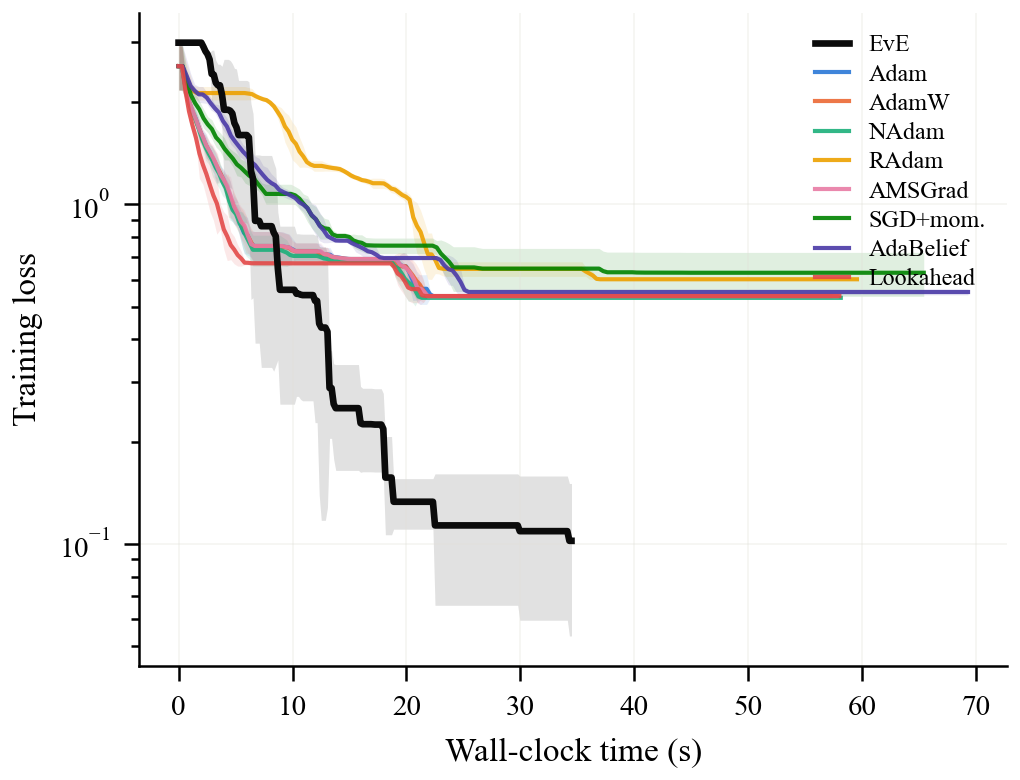}
\caption{Training loss vs.\ wall-clock time.}
\end{subfigure}\\[2pt]
\begin{subfigure}[b]{0.32\linewidth}
\includegraphics[width=\linewidth]{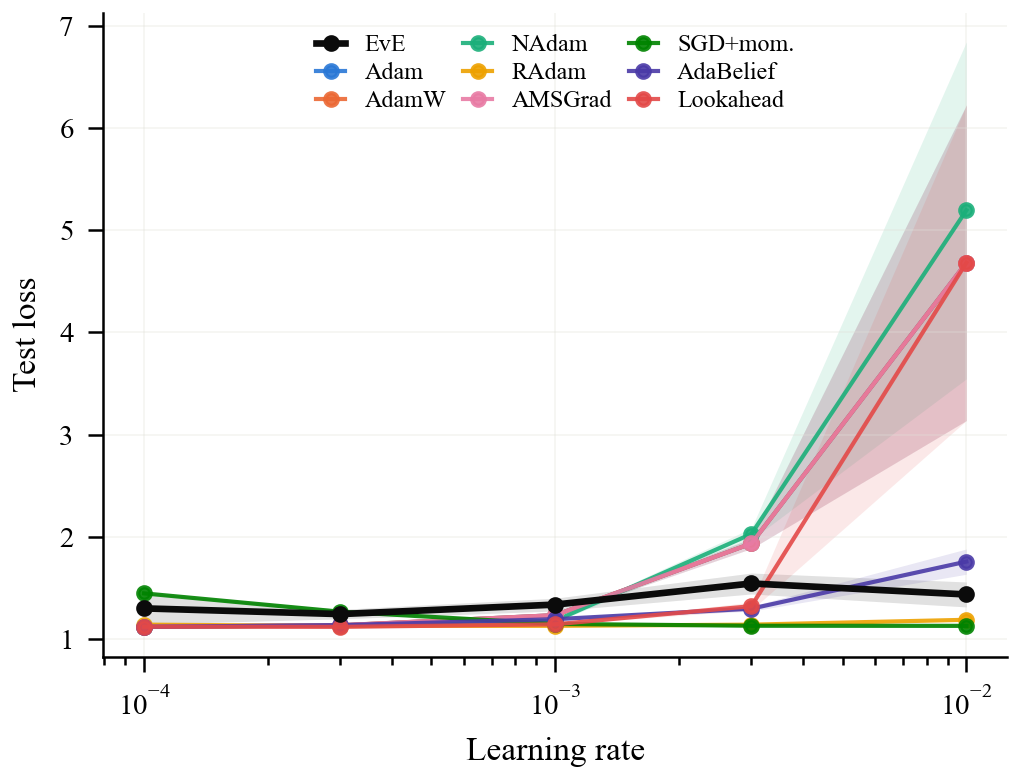}
\caption{Test loss vs.\ learning rate.}
\end{subfigure}\hfill
\begin{subfigure}[b]{0.32\linewidth}
\includegraphics[width=\linewidth]{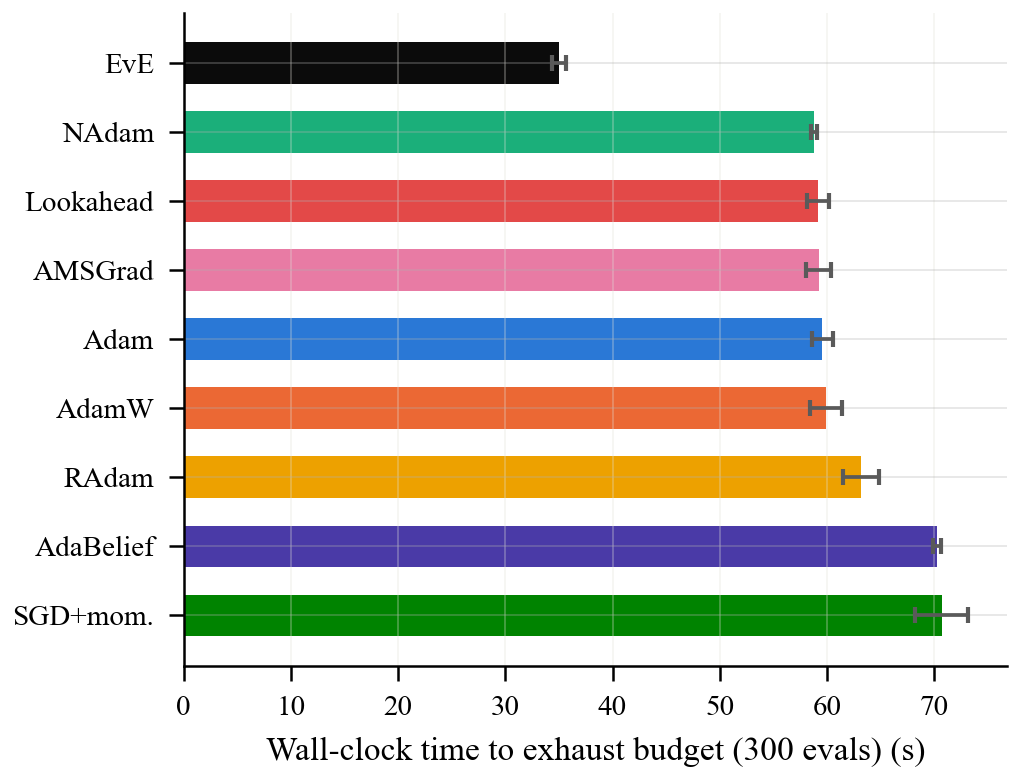}
\caption{Wall-clock time to exhaust the shared budget.}
\end{subfigure}\hfill
\begin{subfigure}[b]{0.32\linewidth}
\includegraphics[width=\linewidth]{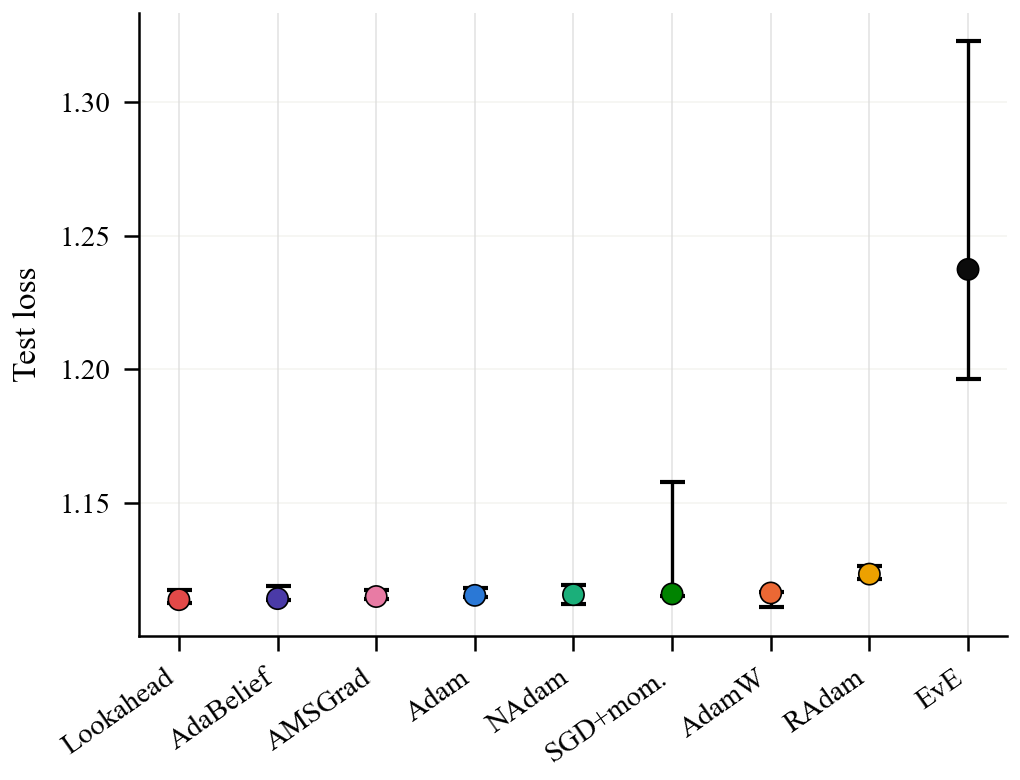}
\caption{Best-LR test loss (median, min-max over 5 seeds).}
\end{subfigure}
\end{minipage}}
\caption{EvE vs.\ eight baselines fine-tuning Qwen2.5-1.5B-Instruct with LoRA on Alpaca, under the same evaluation-cost budget (Section~\ref{sec:setup-exp}), at each method's own best learning rate over 5 seeds. (a)/(b) training runs vs.\ budget and vs.\ wall-clock time. (c) learning-rate sensitivity. (d) wall-clock cost of the shared budget, sorted. (e) final test loss reached (lower is better).}
\label{fig:lora}
\end{figure}
EvE's advantage concentrates in one recognizable regime, and is small or reverses outside it. On the synthetic benchmarks (Table~\ref{tab:scalable-benchmarks}), it wins or ties Adam in $53$ of the $70$ (problem, $n$) cells swept ($76\%$), most decisively on the three multimodal problems (Ackley, Griewank, Rastrigin) and at the extreme dimensions. Adam is strongest on the one convex problem (Sphere) and on Schwefel (from $n=50$); on Rosenbrock the two are statistically indistinguishable at every dimension (Appendix~\ref{app:benchmarks}).

On real neural-network training the pattern shows up differently: across all three tasks (Sections~\ref{sec:mnist}, \ref{sec:lora}, \ref{sec:gsm8k}), EvE finishes the same charged budget $1.7$--$3.9\times$ faster (geometric mean $\sim2.3\times$), a step behind Adam's final quality ($\sim1$ accuracy point on MNIST, $\sim9$--$11\%$ relative test loss on the two LoRA tasks, $5$ accuracy points on GSM8K's full test set; Appendices~\ref{app:gsm8k},~\ref{app:times}). This matches the cost analysis of Section~\ref{sec:theory}: most evaluations go to population upkeep, so the budget buys wall-clock speed, not final precision.

\vspace{-2mm}
\subsection{Practical Considerations}
\label{sec:limitations}
\vspace{-2mm}
A few properties of EvE are worth knowing before using it this way. Its learning rate does not transfer from Adam's and needs its own search ($3$--$100\times$ above the Adam family's depending on the task, Appendices~\ref{app:lora},~\ref{app:gsm8k},~\ref{app:lr-selection}). Its four-member population shows real run-to-run variance -- the widest seed spread of any of the nine methods on both LoRA tasks. Its memory footprint is fixed but non-trivial ($13\times$ the live parameter count against Adam's $4\times$; measured peak is lower, $2.95\times$, Appendix~\ref{app:algo}). Its population size and mutation/initialization scheme are set per task family rather than adapted at run time.

\vspace{-2mm}
\section{Conclusion}
\vspace{-3mm}
We have introduced EvE, a small-population differential evolution optimizer with a gradient-based fallback, evaluated against Adam and seven other baselines across seven synthetic benchmarks, three real neural-network training tasks, and two successive-halving search tasks, under a shared, evaluation-cost-matched budget. EvE trades a modest amount of final quality for a repeatable wall-clock speedup, winning or tying Adam on the benchmarks' multimodal problems and at the extremes of dimension, and finishing real training runs $1.7$--$3.9\times$ faster, a step behind Adam's own final accuracy or loss -- a fast, gradient-aware proxy for the search-heavy step before training a single model. Inside successive halving on UCI Adult, EvE finishes hyperparameter and architecture searches $3.1$--$3.5\times$ faster while ranking configurations about as consistently with Adam as Adam does with itself across seeds. Future work includes adapting its population, mutation, and initialization automatically, and extending the search-loop evidence across datasets.

\subsubsection*{Reproducibility Statement}

We have made a deliberate effort to make every number in this paper reproducible, not just describable. The full algorithm is given as executable pseudocode (Algorithm~\ref{alg:spade}), with every hyperparameter it depends on (population size, DE weight and jitter, PM's distribution index and mutation rate, Adam-fallback learning rate/steps/weight decay) stated as an explicit value rather than left implicit, separately for each experiment family: Table~\ref{tab:scalable-hparams} for the synthetic benchmarks, Table~\ref{tab:mnist-hparams} for MLP/MNIST, and Table~\ref{tab:lora-hparams} for both LoRA experiments (Appendices~\ref{app:benchmarks}--\ref{app:gsm8k}). Every stochastic draw in EvE traces back to a single integer seed through two persistent generators that are never reseeded mid-run (Section~\ref{sec:theory}'s Reproducibility paragraph), and every reported result is aggregated over multiple independent seeds, $21$ for the synthetic benchmarks and $5$ for every neural-network task, stated explicitly rather than left to a single lucky run (the one exception is the single-run budget-extension check in Appendix~\ref{app:gsm8k}, labeled as such there). The evaluation-cost budget that every method is held to, and the exact rule for converting a measured backward pass into that same currency, is given as a closed-form accounting rule (Appendix~\ref{app:eval-cost}), not an informal description, so a re-implementation can reproduce the same stopping condition rather than approximate it. All statistical claims on the synthetic benchmarks use a fully specified test (one-sided Wilcoxon signed-rank, paired by seed) and a fully specified multiple-comparison correction (Holm-Bonferroni, family-wise $\alpha=0.05$ across all $140$ tests), both stated in Table~\ref{tab:scalable-hparams} rather than left to the reader to infer from the table's shading alone. Every dataset split (MNIST's $90\%/10\%$ train/val carve-out, Alpaca's $90\%/5\%/5\%$ train/val/test carve-out, GSM8K's own fixed train/test split plus a train-side validation carve-out) uses a fixed, stated seed shared identically by every method compared, so no method sees an easier or harder split than any other. Together with the seeds and hyperparameters reported here, this is enough to regenerate every result in this paper.

\vspace{-2mm}
\section*{AI Use Statement}
\vspace{-2mm}

AI tools were used in preparing this paper in the following, limited ways. Grammarly was used for fixing typos and grammatical errors in the writing. Claude Sonnet was used for literature survey.

\vspace{-2mm}
\bibliography{iclr2026_conference}
\bibliographystyle{iclr2027_conference}

\newpage
\appendix
\begin{center}
    \huge\bfseries Appendix
\end{center}
\section{Implementation, Memory Cost, and Extended Method Discussion}
\label{app:algo}

\subsection{Evaluation-Cost Accounting}
\label{app:eval-cost}

Section~\ref{sec:setup-exp} holds every method to the same evaluation-cost budget $B$; this subsection gives the exact accounting rule and why it is built this way. Budgeting by evaluation count is standard in evolutionary computation, since a population method has no gradient step or epoch to count, only candidate solutions it asks $f$ to score. That convention was fixed long before anyone paired an EA with a gradient-based fallback for training neural networks, so it says nothing about what a backward pass should cost, and we fill that gap explicitly below rather than silently adopting either paradigm's native accounting.

\paragraph{Counting a forward pass.} On the synthetic benchmarks, where a single call scores a whole batch of $r$ candidate points at once (the population, or a baseline's single trajectory), that call costs exactly $r$: $\text{cost}\gets\text{cost}+r$. On MLP/MNIST, where EvE and every baseline instead call a closure that scores exactly one weight state (the model's current parameters) per call, one call costs exactly $1$. Both are the same currency Section~\ref{sec:theory} calls $C_f$; only the batching differs, not the meaning of one unit of cost.

\paragraph{Counting a backward pass.} A gradient is not free. Whenever either method computes $g=\nabla_z f$ by reverse-mode automatic differentiation (EvE's Adam fallback, Section~\ref{sec:adam-fallback}, and Adam's own step), the wall time $\tau_{\text{bwd}}$ that computation took is measured directly (GPU calls are synchronized before timing, so kernel-launch asynchrony is not mistaken for zero cost) and converted to forward-equivalent units using that run's own measured forward cost:
\begin{equation}
\label{eq:budget-cost}
\hat C_f=\frac{\sum \tau_{\text{fwd}}}{n_{\text{evals}}},\qquad \text{cost}\gets\text{cost}+\frac{\tau_{\text{bwd}}}{\hat C_f},
\end{equation}
where $n_{\text{evals}}$ and $\sum\tau_{\text{fwd}}$ are that run's forward-pass count and total measured forward time so far. $\hat C_f$ is never assumed (e.g. as a fixed $c_\nabla$ multiplier); it is measured on the same hardware, problem, and dimension $n$ as the run it charges, and each backward's contribution is added once and never revisited, so cost is monotone non-decreasing by construction. A run terminates the first time $\text{cost}\ge B$; we use $B=2{,}000$ for the synthetic benchmarks and $B=3{,}000$ for MLP/MNIST (Appendix~\ref{app:mnist}).

\paragraph{Why this matters.} EvE's Adam fallback and vanilla Adam both spend gradients, but at different rates: Adam takes one gradient step every evaluation, EvE only when its DE/PM candidate fails to improve on the incumbent (Section~\ref{sec:adam-fallback}). Charging forward passes alone would let the method that backpropagates more often run for free relative to the one that does not; the accounting above closes that gap for both methods identically, on the same hardware and problem rather than a literature-assumed ratio.

\subsection{Vectorized Updates and Memory Cost}

Within the steady-state loop, the per-iteration operations (the DE combination of line 8 and PM's coordinate write) touch only the single challenger being constructed that iteration, each as one vectorized elementwise op over its $n$ coordinates, so their dispatch overhead is $O(1)$ per iteration regardless of $n_{\text{pop}}$; only $O(n)$ (DE) or $O(k)$ (PM, $k=1$) arithmetic actually scales with the problem. The one-time proximity initialization (line 1) instead loops once over the $n_{\text{pop}}-1$ extra population members, a fixed, constant-size cost paid once per run, not per iteration. Evaluating a candidate substitutes its coordinates directly into $f$'s own input (a model's parameters, or a flat vector) before calling $f$, so nothing about how $f$ itself is computed is changed. The Adam fallback keeps one persistent pair of moment vectors $(m,v)\in\mathbb{R}^n\times\mathbb{R}^n$ per population slot, allocated once, rather than fresh buffers every burst.

This is also where the memory cost of Section~\ref{sec:theory} comes from. Beyond the $1\times$ live parameter copy being optimized, the population itself holds $n_{\text{pop}}=4$ further full copies ($4\times$), and the $n_{\text{pop}}=4$ persistent Adam states each hold two moment vectors ($2\times$ per slot, $8\times$ total), for $4+8=12\times$ extra storage, i.e. $13\times$ total. A plain Adam optimizer needs only its $1\times$ live parameters, $1\times$ gradient, and $2\times$ moments, i.e. $4\times$. The multiplier is fixed by $n_{\text{pop}}=4$ and does not grow with $n$.

\paragraph{Measured memory versus this theoretical count.} The $13\times$-versus-$4\times$ figures above count only parameter-and-optimizer-state storage; they say nothing about total GPU memory in a real run, which also includes activations, PyTorch's caching-allocator overhead, and CUDA context memory, none of which scale with $n_{\text{pop}}$. To check how much of a gap that leaves, we measured actual peak GPU memory on the MLP/MNIST setup (Section~\ref{sec:mnist}): starting a fresh CUDA context per method (so one method's cached-but-freed memory never leaks into the other's reading), building the model, resetting PyTorch's peak-memory counter (\texttt{torch.cuda.reset\_peak\_memory\_stats}), then running $200$ real training iterations (Adam: $200$ gradient steps; EvE: $200$ steady-state iterations, enough that every one of its $4$ population slots' Adam-fallback optimizer fired and populated its full moment state, verified directly rather than assumed) and reading \texttt{torch.cuda.max\_memory\_allocated}/\texttt{max\_memory\_reserved} at the end. Figure~\ref{fig:memory} reports the result: EvE's measured peak allocated memory is $2.95\times$ Adam's ($213.95$MB versus $72.43$MB; $2.86\times$ on peak reserved memory), noticeably below the $3.25\times$ the parameter-storage-only count in this appendix would predict ($13/4$). The gap is consistent with the explanation above: a fixed, model-independent slice of GPU memory (activations for this MLP's minibatch, allocator bookkeeping) is shared identically by both methods and does not grow with EvE's population, diluting the ratio actual total memory shows relative to the ratio parameter storage alone would show. We report this as a real, measured correction to the theoretical multiplier, not a replacement for it: the $13\times$/$4\times$ count in this appendix is still the correct answer to ``how much more parameter-and-optimizer-state storage does EvE hold,'' just not the correct answer to ``how much more total GPU memory does a real EvE run use,'' which is smaller.

\begin{figure}[h]
\centering
\includegraphics[width=0.55\linewidth]{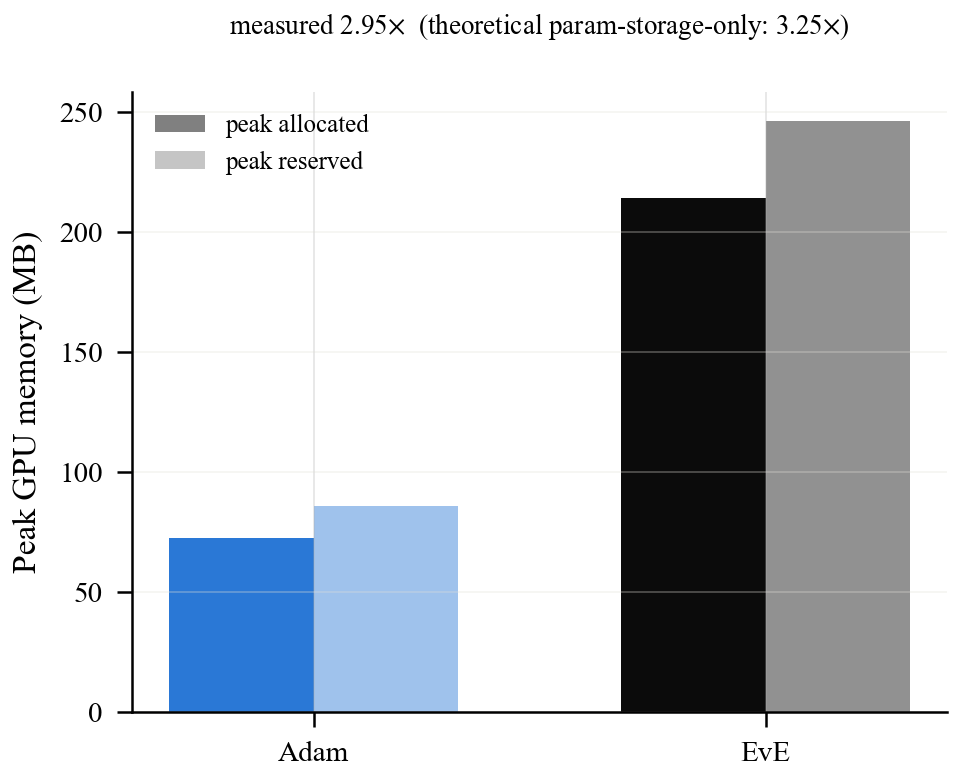}
\caption{Measured peak GPU memory, EvE vs.\ Adam, on MLP/MNIST (Section~\ref{sec:mnist}), $200$ real training iterations per method, each in its own CUDA context. The measured ratio ($2.95\times$ peak allocated) sits below the parameter-storage-only theoretical ratio ($3.25\times$) derived above.}
\label{fig:memory}
\end{figure}

\subsection{How the Batch Advances}
\label{app:batch}

When $f$ is a minibatch loss $f(x;\mathcal{B})=\frac{1}{|\mathcal{B}|}\sum_{\xi\in\mathcal{B}}\ell(x;\xi)$, one new minibatch $\mathcal{B}_t$ is drawn exactly once, at the very start of each iteration $t$ (line 6 of Algorithm~\ref{alg:spade}); every evaluation for the rest of that iteration, including every step inside a fallback burst, is scored on that same $\mathcal{B}_t$. Only the challenged slot's record is refreshed on $\mathcal{B}_t$ at the start of the iteration, $f_i^{(t)}\!\gets\! f(x_i;\mathcal{B}_t)$; the other three slots keep whatever value they were last scored at (a possibly different, older batch) until their own turn comes around, so the extra cost of this refresh is one evaluation per iteration, not $n_{\text{pop}}$. Because $\mathrm{Var}_{\mathcal{B}}\!\big[f(x;\mathcal{B})-f(x';\mathcal{B})\big]\ll \mathrm{Var}_{\mathcal{B},\mathcal{B}'}\!\big[f(x;\mathcal{B})-f(x';\mathcal{B}')\big]$ for nearby $x,x'$, pairing the child, the fallback burst, and the slot's own refreshed value on the same $\mathcal{B}_t$ (common random numbers) keeps the greedy comparison dominated by the true fitness gap rather than sampling noise, and lets an $n_{\text{adam}}$-step fallback burst train on one fetched batch instead of paying a fresh data-fetch every step.

\subsection{Why PM Helps the Benchmarks but not Neural Network Training}
\label{app:pm}

PM (Section~\ref{sec:mutation}) perturbs exactly $k=1$ coordinate of $z_{\text{child}}$, chosen using the classical per-gene mutation rate $p=1/n$, and is guarded by its own accept/reject gate: the mutant only replaces $z_{\text{child}}$ if it scores better on the same batch. Because of that gate, PM cannot make $z_{\text{child}}$ worse; the only real question is whether it is worth its one extra evaluation, which depends on how often that single-coordinate mutant actually wins.

On the synthetic benchmarks in Table~\ref{tab:scalable-benchmarks}, several problems are explicitly multimodal ($f$ has many local minima whose basins are separated by regions where $\nabla f$ points away from the global optimum). A gradient step from $z_{\text{child}}$ is only locally correct in that setting, so an undirected, single-coordinate nudge has a real, non-negligible chance of landing in a different, better basin and winning the gate. A neural network's training loss, in contrast, is a smoothed average over a minibatch, $f(x;\mathcal{B})=\frac{1}{|\mathcal{B}|}\sum_{\xi\in\mathcal{B}}\ell(x;\xi)$, in $n$ that reaches into the millions: changing one coordinate out of millions is a vanishing perturbation relative to the curvature the gradient already resolves almost everywhere, so the mutant essentially never beats the clean child. In that regime PM's accept/reject gate almost always rejects, and the run pays for an extra evaluation with no return. This is why EvE runs PM only on the benchmarks and skips it for neural network training, going directly from the DE child to the Adam fallback comparison of Section~\ref{sec:adam-fallback}.

\subsection{Learning-Rate Selection: Test Versus Validation}
\label{app:lr-selection}

Wherever a method is reported at its own best learning rate (MLP/MNIST, LoRA/Alpaca, LoRA/GSM8K), the rate is chosen from the shared five-value grid by the mean held-out test metric over the five seeds (test accuracy on MNIST, test loss on the two LoRA tasks), with the same rule for every method. Choosing on test data makes absolute numbers slightly optimistic for every method, so we also computed, from the same saved runs, what changes if the rate is chosen on the validation metric instead (validation accuracy on MNIST, validation loss on the LoRA tasks). Table~\ref{tab:lr-selection} summarizes the comparison.

\begin{table}[h]
\centering
\caption{Effect of choosing each method's learning rate on validation instead of test data, from the same saved runs.}
\label{tab:lr-selection}
\begin{tabular}{lcccc}
\hline
Task & Methods with same rate & EvE rate (test / val.) & EvE result (test / val.) & Largest baseline change \\
\hline
MLP/MNIST & $5$ of $9$ & $10^{-3}$ / $10^{-3}$ & $97.03\%$ / $97.03\%$ & $0.12$ points \\
LoRA/Alpaca & $7$ of $9$ & $3\times10^{-4}$ / $3\times10^{-4}$ & $1.240$ / $1.240$ & $0.002$ \\
LoRA/GSM8K & $7$ of $9$ & $10^{-2}$ / $10^{-4}$ & $0.640$ / $0.654$ & $0.0004$ \\
\hline
\end{tabular}
\end{table}

On MNIST and Alpaca, EvE's own rate and result do not change, and no baseline moves by more than the amount shown, so the comparison is stable. The exception is EvE on GSM8K: chosen on test loss its best rate is $10^{-2}$, chosen on validation loss it is $10^{-4}$ (the same as seven of the baselines), and its test loss at the validation-chosen rate is $0.654\pm0.019$ instead of $0.640\pm0.016$, which widens the gap to the baselines' cluster (means $0.585$--$0.590$) from about $9\%$ to about $11\%$. Two statements in this paper therefore depend on the selection rule and hold only under test-based selection: that EvE's GSM8K rate is two orders of magnitude above the baselines', and that the GSM8K loss gap is smaller than Alpaca's (under validation-based selection both are about $11\%$). We report test-based selection in the main results because it was fixed before these checks and applied identically to every method.

\subsection{Training-Time Summary}
\label{app:times}

Table~\ref{tab:times} collects the wall-clock time each experiment spends on the shared, evaluation-cost-matched budget, in seconds, as mean $\pm$ std over $5$ seeds at each method's own best learning rate (EvE against the range over the eight baselines). The UCI Adult rows are the total wall-clock time of a full successive-halving sweep.

\begin{table}[h]
\centering
\caption{Wall-clock training time in seconds under the shared budget, EvE against the eight baselines (range).}
\label{tab:times}
\begin{tabular}{lccc}
\hline
Task & EvE (s) & Baselines (s) & Speedup \\
\hline
MLP/MNIST ($B=3{,}000$) & $49.3$ & $182.3$--$192.3$ & $3.7$--$3.9\times$ \\
LoRA/Alpaca ($B=300$) & $35.0\pm0.6$ & $58.8$--$70.7$ & $1.7$--$2.0\times$ \\
LoRA/GSM8K ($B=900$) & $100.4\pm4.7$ & $193.8$--$198.7$ & $1.9$--$2.0\times$ \\
UCI Adult, hyperparameter search & $21.8\pm0.3$ & $73.6$--$76.8$ & $3.4$--$3.5\times$ \\
UCI Adult, architecture search & $21.5\pm0.9$ & $66.9$--$75.8$ & $3.1$--$3.5\times$ \\
\hline
\end{tabular}
\end{table}

\section{Synthetic Benchmarks: Parameters and Detailed Analysis}
\label{app:benchmarks}

\paragraph{Parameters.} Table~\ref{tab:scalable-hparams} lists every setting behind the Table~\ref{tab:scalable-benchmarks} result in the main text. Two settings differ from EvE's neural-network default (Table~\ref{tab:mnist-hparams}): population initialization, chosen by the informal spot check described below, and the Adam-fallback learning rate and step count ($0.05$ and $7$), which are the defaults of an earlier flat-vector implementation of EvE rather than values swept on this suite. The Adam baseline uses a single fixed learning rate of $0.1$ for every problem and dimension, chosen once and not tuned. Neither method's learning rate was tuned per problem, and we did not run an Adam learning-rate sweep on this suite, so an individual cell could change under a different Adam rate. One deviation from the textbook problem definitions: Sphere is recentered to $\sum_i(x_i-0.5)^2$ on $[0,1]^n$ rather than $\sum_i x_i^2$ on a domain centered at the origin, so its optimum sits at the interior midpoint of the search box used here instead of on its boundary; this changes where the optimum is, not the problem's difficulty (still the one convex, unimodal problem in the suite), and both methods are swept over the identical recentered objective.

\begin{table}[h]
\centering
\caption{Shared parameters for the synthetic scalable-benchmark sweep (Table~\ref{tab:scalable-benchmarks}).}
\label{tab:scalable-hparams}
\begin{tabular}{ll}
\hline
Setting & Value \\
\hline
Evaluation-cost budget $B$ & $2{,}000$ (Section~\ref{sec:setup-exp} accounting) \\
Problem dimensions $n$ & $\{2,10,20,50,100,200,500,1{,}000,10^5,10^6\}$ \\
Seeds & $\{0,\dots,20\}$ (21 independent runs per cell) \\
EvE population size $n_{\text{pop}}$ & $4$ (fixed, Section~\ref{sec:setup}) \\
Population initialization & Latin Hypercube Sampling (see below), not proximity (Section~\ref{sec:init}) \\
DE weight $F$ & $0.5$, jitter $\gamma=10^{-4}$ (Section~\ref{sec:de}) \\
EvE polynomial mutation & on; $\eta=20$, rate $p=1/n$ ($k=1$ gene), accept/reject gate (Appendix~\ref{app:pm}) \\
EvE Adam-fallback learning rate & $0.05$ \\
EvE Adam-fallback steps $n_{\text{adam}}$ & $7$ \\
EvE Adam-fallback weight decay & $10^{-4}$ (L2, Section~\ref{sec:adam-fallback}) \\
Adam baseline learning rate & $0.1$, fixed for every problem and dimension (not tuned) \\
Significance test & one-sided Wilcoxon signed-rank per cell, paired by seed \\
Multiple-comparison correction & Holm-Bonferroni, family-wise $\alpha=0.05$, across all $140$ tests \\
\hline
\end{tabular}
\end{table}

\paragraph{Latin Hypercube population initialization.} Section~\ref{sec:init}'s proximity initialization is tuned for neural-network training, where the population should start as a tight cluster the Adam fallback can immediately refine. On these low-dimensional, closed-form benchmarks we instead initialize the population with Latin Hypercube Sampling (LHS): each coordinate's range is divided into $n_{\text{pop}}$ equal strata, one point is drawn per stratum, and strata are independently permuted per coordinate, so the four initial points are spread across the domain rather than clustered around one anchor. In an informal spot check at matched budget on a small subset of (problem, $n$) cells, LHS initialization outperformed both proximity initialization and a fully independent uniform draw on most of the cells tried, so it is the default used throughout Table~\ref{tab:scalable-benchmarks}; we did not run this comparison at the full scale of Table~\ref{tab:scalable-benchmarks} itself, so it should be read as the reason for our choice of default, not as an independently validated result. Proximity initialization remains the default for neural network training (Section~\ref{sec:mnist}), where it was separately validated.

\paragraph{Per-problem analysis.} Ackley, Griewank, and Rastrigin are the three most explicitly multimodal problems swept, and EvE wins or ties every cell on all three, most decisively at $n\ge10^5$: on Ackley at $n=10^6$, EvE reaches a median of $6.68$ against Adam's $19.6$, and on Griewank at the same scale, $3.71\times10^6$ against $1.60\times10^7$. Both are consistent with the population/mutation machinery finding escape routes a single gradient trajectory does not, and with Adam's own gradient becoming less informative as $n$ grows on these landscapes (Section~\ref{sec:benchmarks}). Rosenbrock is the one problem where EvE's advantage all but disappears: every cell is statistically indistinguishable, reflecting its smooth, curved-valley structure, where a directed gradient step is already close to the best available move and an undirected mutation has little to add. Schwefel is genuinely mixed: EvE wins only at $n=2$, ties at $n=10$ and $20$, and Adam pulls ahead from $n=50$ upward; this problem's global optimum, sitting far from the search origin, appears to reward Adam's directed descent more than population diversity once the search space is large enough for that distance to matter. Sphere, the one unimodal, convex problem in the suite, is where Adam is most competitive: Adam is significantly better at every $n$ (at $n\le10$ both reach numerical zero, with Adam exactly zero), exactly the regime where a pure gradient method should have no weakness to exploit. Zakharov is the most extreme case at large $n$: Adam's own iterate diverges numerically at $n\ge10^5$ (median losses of $10^{29}$ and $10^{36}$), while EvE's greedy, bounded population never leaves a finite, bounded region, so its advantage there reflects avoiding a numerical failure mode as much as finding a better optimum.

\begin{table}[h]
\centering
\caption{Per-problem summary of EvE vs.\ Adam on the synthetic benchmarks: over the ten dimensions swept, the number of (problem, $n$) cells where EvE is significantly better, where the two are statistically indistinguishable, and where Adam is significantly better (one-sided Wilcoxon signed-rank, paired by seed, Holm-Bonferroni corrected, family-wise $\alpha=0.05$; $21$ seeds per cell). The full min/median/max at every cell is Table~\ref{tab:scalable-benchmarks} in the main text.}
\label{tab:scalable-summary}
\begin{tabular}{lcccl}
\hline
Problem & EvE better & Indistinguishable & Adam better & Where Adam wins \\
\hline
Ackley & 7 & 3 & 0 & -- \\
Griewank & 10 & 0 & 0 & -- \\
Rastrigin & 7 & 3 & 0 & -- \\
Rosenbrock & 0 & 10 & 0 & -- \\
Schwefel & 1 & 2 & 7 & $n\ge50$ \\
Sphere & 0 & 0 & 10 & all $n$ \\
Zakharov & 8 & 2 & 0 & -- \\
\hline
All problems & 33 & 20 & 17 & \\
\hline
\end{tabular}
\end{table}

\section{MLP/MNIST: Architecture and Hyperparameters}
\label{app:mnist}

\paragraph{Architecture.} A plain fully connected network, no convolutions: a $784$-unit input (one $28\times28$ MNIST image, flattened, pixel values as-is), two hidden layers of $1{,}024$ units each with ReLU activations, and a $10$-unit output layer scored with cross-entropy. Table~\ref{tab:mnist-arch} gives the per-layer shapes; the total is $1{,}863{,}690$ trainable weights. Every method (EvE and all eight baselines) trains this exact same architecture, initialized from the exact same random draw for a given seed, so architecture and initialization are never a source of difference between methods.

\begin{table}[h]
\centering
\caption{MLP architecture. $\text{fan\_in}$ is the preceding layer's width; biases inherit their weight tensor's fan-in.}
\label{tab:mnist-arch}
\begin{tabular}{lccc}
\hline
Layer & Shape & Activation & Parameters \\
\hline
Input & $784$ & -- & -- \\
Hidden 1 & $784\to1{,}024$ & ReLU & $802{,}816+1{,}024$ (weight+bias) \\
Hidden 2 & $1{,}024\to1{,}024$ & ReLU & $1{,}048{,}576+1{,}024$ \\
Output & $1{,}024\to10$ & -- & $10{,}240+10$ \\
\hline
Total & & & $1{,}863{,}690$ \\
\hline
\end{tabular}
\end{table}

\paragraph{Initialization and bounds.} Every weight tensor is initialized, and every method's search/step bounds are set, from the same Kaiming-derived bound $b=\sqrt{6/\text{fan\_in}}$ per tensor (a bias inherits its own weight tensor's fan-in), the same convention Section~\ref{sec:setup} uses by default for any network EvE is not told an explicit domain for. This keeps initialization and any bound-dependent clamping identical across EvE and every baseline.

\paragraph{Data.} MNIST: $60{,}000$ training images, $10{,}000$ held-out test images, no augmentation. The training set is further split $90\%/10\%$ into a training and a validation set with a fixed split seed ($42$), shared identically by every method. Minibatch size is $256$; each reported test accuracy is the accuracy, on the full $10{,}000$-image test set, of the weight snapshot that reached the lowest validation loss during that run (not simply the final weights), so early stopping does not favor any one method.

\paragraph{Hyperparameters and protocol.} Table~\ref{tab:mnist-hparams} lists the shared sweep. Every method is run at every (learning rate, seed) pair, and Section~\ref{sec:mnist}'s numbers are reported at each method's own best learning rate (highest mean test accuracy across seeds), never a single shared learning rate. All methods are stopped at the same evaluation-cost budget $B=3{,}000$ using the accounting of Section~\ref{sec:setup-exp}, so the comparison is at matched charged work, not matched step count.

\begin{table}[h]
\centering
\caption{Shared MLP/MNIST training protocol.}
\label{tab:mnist-hparams}
\begin{tabular}{ll}
\hline
Setting & Value \\
\hline
Evaluation-cost budget $B$ & $3{,}000$ (Section~\ref{sec:setup-exp} accounting) \\
Learning-rate grid & $\{10^{-4},3\times10^{-4},10^{-3},3\times10^{-3},10^{-2}\}$ \\
Seeds & $\{0,1,2,3,4\}$ (5 independent runs per cell) \\
Minibatch size & $256$ \\
Weight decay & $5\times10^{-4}$ (decoupled for AdamW; L2 for SGD+momentum and EvE's Adam fallback) \\
SGD momentum & $0.9$ \\
EvE population size $n_{\text{pop}}$ & $4$ (fixed, Section~\ref{sec:setup}) \\
EvE Adam-fallback steps $n_{\text{adam}}$ & $3$ (default, Section~\ref{sec:adam-fallback}) \\
EvE polynomial mutation & off (Section~\ref{sec:mutation}) \\
\hline
\end{tabular}
\end{table}

\paragraph{Baselines.} Adam, AdamW, NAdam, RAdam, and AMSGrad \citep{reddi2018amsgrad} (Adam with its AMSGrad variant enabled) are standard optimizer implementations; AdaBelief \citep{zhuang2020adabelief} and Lookahead(Adam) \citep{zhang2019lookahead} come from a widely used third-party optimizer collection; SGD uses (Nesterov-free) momentum $0.9$. None of the eight baselines receives any tuning beyond the shared learning-rate grid above: no baseline-specific learning-rate schedule, warmup, or extra regularization.

\paragraph{Detailed discussion of Figure~\ref{fig:mnist}.} Panel (a) shows the same training runs from two angles. Its left pane plots training loss against the evaluation-cost budget both methods are held to equally; here EvE tracks in the middle of the pack, a bit behind the Adam family but ahead of plain SGD. Its right pane plots the identical runs against wall-clock time instead, and this is where the difference shows up: EvE has already used its whole budget and stopped by around 49 seconds, while every baseline is still going at 182 to 192 seconds to spend that same budget, a roughly $3.7$--$3.9\times$ difference in wall-clock time for identical charged work. Panel (c) shows this same time-to-exhaust-budget comparison as a bar chart. Panel (d) shows what each method's own best learning rate actually reaches: EvE lands around 97.0\% test accuracy, close to plain SGD with momentum, while the Adam-family methods cluster tighter around 98.0-98.2\%, a gap of roughly one accuracy point. Panel (b) shows how sensitive each method is to the learning-rate choice; EvE's curve sits between SGD's steep dependence and the Adam family's flat, forgiving one.

\section{UCI Adult: Hyperparameter and Architecture Search}
\label{app:hpo}

\paragraph{Successive halving.} Both sweeps use synchronous successive halving (SHA) \citep{jamieson2016non}: $N_{\text{configs}}=48$ configurations are sampled once (log-uniform learning rate; log-uniform weight decay for the hyperparameter sweep, or discrete width/depth choices for the architecture sweep), then trained in rungs at evaluation-cost budgets $100, 300, 900, 2{,}700, 3{,}000$ (reduction factor $\eta=3$, matching $B=3{,}000$ used nowhere else in this paper but chosen for the same reason MNIST used $B=3{,}000$: enough budget for a real MLP to separate good configurations from bad ones). At each rung boundary only the top third of surviving configurations (by validation accuracy so far) is promoted to the next, larger budget; eliminated configurations are never trained further. Every method's training is \emph{resumed}, not restarted, across rungs -- the same model, optimizer, and budget tracker objects persist, with the tracker's budget simply raised before re-entering the training loop -- so a promoted configuration's reported wall-clock time is the honest sum of the budget it actually consumed, never double-counted. We label this SHA rather than ASHA (asynchronous SHA): the asynchronous/synchronous distinction only matters with multiple parallel workers, and everything here runs sequentially on one GPU, where the two are identical in effect.

\paragraph{Dataset.} UCI Adult Income \citep{adult1996}: $32{,}561$ training rows, $16{,}281$ held-out test rows (the dataset's own fixed split), binary target (annual income above/below \$50K). The \texttt{fnlwgt} column (a census sampling weight, not predictive of income) is dropped, standard practice for this dataset. Categorical columns are one-hot encoded (missing values marked ``?'' kept as their own category, not dropped, so train and test rows stay aligned); numeric columns are standardized using training-set statistics only, applied unchanged to validation and test, to avoid leakage. This yields $107$ input features. A fixed $10\%$ validation carve-out (seed $42$, same convention as every other experiment in this paper) is taken from the training rows.

\paragraph{Why UCI Adult?} Unlike MNIST or the LoRA tasks, this is chosen specifically as a widely-cited, standard AutoML/HPO benchmark (used throughout Auto-sklearn, AutoGluon, FT-Transformer, and most of the HPO literature), not a dataset constructed to favor EvE, and tabular classification is one of the more hyperparameter-sensitive corners of the field (default hyperparameters are famously mediocre without tuning), making ``does a faster inner optimizer speed up the search'' a meaningful question here.

\subsection{Hyperparameter Search: Learning Rate and Weight Decay}
\label{app:hpo-lr}

Architecture is fixed (a $2$-hidden-layer, $256$-unit MLP); the search is over learning rate (log-uniform, $[10^{-4},10^{-2}]$) and weight decay (log-uniform, $[10^{-6},10^{-2}]$), identical configurations for every method (fixed seed $0$ for config sampling). Table~\ref{tab:hpo-full} reports, for each method, the total wall-clock time to complete its own 5-seed SHA sweep and the best validation accuracy found, both as mean $\pm$ std over the 5 seeds.

\begin{table}[h]
\centering
\caption{UCI Adult, learning-rate/weight-decay SHA sweep, mean $\pm$ std over 5 seeds, sorted by wall-clock time.}
\label{tab:hpo-full}
\begin{tabular}{l c c}
\hline
Method & Time (s) & Best Val Acc \\
\hline
\textbf{EvE} & 21.8 $\pm$ 0.3 & 0.8544 $\pm$ 0.0018 \\
SGD+mom. & 73.6 $\pm$ 1.3 & 0.8549 $\pm$ 0.0020 \\
NAdam & 75.1 $\pm$ 0.5 & 0.8591 $\pm$ 0.0018 \\
RAdam & 75.5 $\pm$ 0.9 & 0.8584 $\pm$ 0.0009 \\
AdaBelief & 75.6 $\pm$ 0.8 & 0.8579 $\pm$ 0.0010 \\
Adam & 75.7 $\pm$ 0.9 & 0.8582 $\pm$ 0.0019 \\
AdamW & 75.7 $\pm$ 0.6 & 0.8587 $\pm$ 0.0010 \\
AMSGrad & 76.2 $\pm$ 0.4 & 0.8580 $\pm$ 0.0010 \\
Lookahead & 76.8 $\pm$ 1.0 & 0.8587 $\pm$ 0.0019 \\
\hline
\end{tabular}

\end{table}

EvE completes the sweep in $21.8\pm0.3$s against $73.6$--$76.8$s for the eight baselines ($3.4$--$3.5\times$ faster), reaching a best validation accuracy of $0.8544\pm0.0018$ against the baselines' $0.8549$--$0.8591$ (a gap of $0.05$--$0.47$ points, smaller than the gap reported for MNIST). EvE was faster than every one of the eight baselines on all $5$ seeds, a fully consistent trend; we do not report a formal significance test on it, since at $n=5$ seeds even a perfectly consistent difference sits at the mechanical floor of what a Wilcoxon signed-rank test can express and would not survive the Holm-Bonferroni correction used for Table~\ref{tab:scalable-benchmarks}'s $21$-seed comparisons, so a $p$-value here would overstate the evidence rather than clarify it.

\subsection{Architecture Search: Width and Depth}
\label{app:hpo-nas}

The search space here is genuinely architectural: hidden width $H\in\{64,128,256,512\}$ and depth $\in\{1,2\}$ hidden layers, with learning rate swept alongside (same range as Appendix~\ref{app:hpo-lr}) but weight decay held fixed at $5\times10^{-4}$, so the experiment isolates architecture choice rather than repeating the hyperparameter sweep with extra noise. Each SHA configuration therefore trains a genuinely different-shaped model, not just a different-hyperparameter run of a fixed one.

\begin{table}[h]
\centering
\caption{UCI Adult, architecture (width/depth) SHA sweep, mean $\pm$ std over 5 seeds, sorted by wall-clock time.}
\label{tab:nas-full}
\begin{tabular}{l c c}
\hline
Method & Time (s) & Best Val Acc \\
\hline
\textbf{EvE} & 21.5 $\pm$ 0.9 & 0.8544 $\pm$ 0.0010 \\
AdaBelief & 66.9 $\pm$ 1.2 & 0.8587 $\pm$ 0.0017 \\
Lookahead & 68.2 $\pm$ 1.1 & 0.8588 $\pm$ 0.0006 \\
SGD+mom. & 69.6 $\pm$ 0.9 & 0.8533 $\pm$ 0.0014 \\
AMSGrad & 69.8 $\pm$ 2.1 & 0.8588 $\pm$ 0.0017 \\
RAdam & 70.5 $\pm$ 1.7 & 0.8585 $\pm$ 0.0012 \\
NAdam & 70.6 $\pm$ 1.3 & 0.8582 $\pm$ 0.0013 \\
AdamW & 71.8 $\pm$ 1.5 & 0.8589 $\pm$ 0.0008 \\
Adam & 75.8 $\pm$ 3.7 & 0.8585 $\pm$ 0.0013 \\
\hline
\end{tabular}

\end{table}

EvE completes the architecture sweep in $21.5\pm0.9$s against $66.9$--$75.8$s ($3.1$--$3.5\times$ faster), reaching $0.8544\pm0.0010$ against the baselines' $0.8533$--$0.8589$; EvE beats sgd\_momentum's best-found accuracy outright (by $0.11$ points) and otherwise trails the remaining seven baselines by $0.38$--$0.45$ points, the same modest-quality-for-speed pattern as the rest of this paper. The same caveat from Appendix~\ref{app:hpo-lr} applies here: EvE was faster on every seed against every baseline, but we do not attach a $p$-value to that at this seed count for the same reason given there.

One genuinely new finding this search space can show that a hyperparameter-only search cannot: which architecture each method actually selected. Width $H=512$ was the single most frequently selected choice across all 9 methods' 5 seeds (chosen in $30$ of the $45$ winning configurations), and depth $2$ in $30$ of $45$; restricted to the eight baselines alone the preference is even stronger ($28$ of $40$ for both $H=512$ and depth $2$), a clear majority converging on the largest, deepest network in the search space being the best one found. EvE's own 5 seeds diverge from that majority: $H=512$ in only 2 of 5 runs ($H=256$ in the other 3) and depth $2$ in only 2 of 5 (depth $1$ in the other 3) -- despite this, EvE's accuracy remains close to the baseline cluster's (Table~\ref{tab:nas-full}). So unlike what an earlier, leakage-affected run of this sweep suggested, EvE does not simply rediscover the same architecture the gradient-only baselines converge on; it settles on a smaller, shallower network more often than not, and reaches comparable accuracy anyway, a genuine (if modest) difference in what the search converges on, not only in how fast it gets there.

\subsection{Does EvE Rank Configurations Like Adam Does?}
\label{app:rank-agreement}

A fast proxy is only useful if it orders candidate configurations the way the trainer it stands in for would. Every method in Appendices~\ref{app:hpo-lr} and~\ref{app:hpo-nas} trains the same $48$ sampled configurations (the configuration-sampling seed is fixed), so this can be measured directly. At the first SHA rung (evaluation-cost budget $100$), we compute Kendall's $\tau$ (tie-corrected) between EvE's validation-accuracy ranking of the $48$ configurations and Adam's, pairing runs by training seed, and report mean $\pm$ std over the $5$ seeds. Two references put the values in context: the same statistic between Adam's rankings under two different training seeds (the agreement a method has with itself, set by training noise), and between Adam and AdamW (a near-identical optimizer).

\begin{table}[h]
\centering
\caption{Rank agreement (Kendall's $\tau$) at the first SHA rung over the same $48$ configurations, mean $\pm$ std over $5$ seeds.}
\label{tab:rank-agreement}
\begin{tabular}{lcc}
\hline
Comparison & Hyperparameter search & Architecture search \\
\hline
EvE vs.\ Adam & $0.663\pm0.093$ & $0.691\pm0.032$ \\
EvE vs.\ AdamW & $0.677\pm0.063$ & $0.709\pm0.041$ \\
Adam vs.\ Adam (different training seeds) & $0.694\pm0.061$ & $0.670\pm0.029$ \\
Adam vs.\ AdamW & $0.714\pm0.078$ & $0.848\pm0.018$ \\
\hline
\end{tabular}
\end{table}

EvE's ranking agrees with Adam's about as well as Adam's agrees with itself under a different training seed ($0.663$ against $0.694$ on the hyperparameter search, $0.691$ against $0.670$ on the architecture search; Table~\ref{tab:rank-agreement}), so the disagreement is comparable to training noise. EvE's own top-ranked configuration at this rung sat at rank $4.8$ of $48$ under Adam on average in the hyperparameter search (per seed: $4,5,3,10,2$) and $6.8$ in the architecture search ($5,4,3,16,6$). Three limits apply. The agreement is not perfect ($\tau<1$), so a pair of configurations can be ordered differently by EvE and Adam, at a rate comparable to Adam's own run-to-run reordering. It is measured on one dataset and at the first rung only, the early-ranking regime this paper's claim concerns, and does not test agreement at later rungs. And on the architecture search Adam agrees more closely with AdamW ($0.848$) than with EvE ($0.691$), so EvE is not as faithful to Adam as a near-identical optimizer is.

\section{LoRA/Alpaca: Architecture and Hyperparameters}
\label{app:lora}

\paragraph{Model and adapters.} Qwen2.5-1.5B-Instruct, base weights frozen, LoRA adapters (rank $32$, alpha $64$, dropout $0$) on the query, value, and output projections plus the MLP, matching torchtune's own \texttt{qwen2\_5/1.5B\_lora\_single\_device} reference recipe for this model. Only the LoRA adapter parameters are ever trainable or optimized by any method; the base model is loaded once per run in bfloat16, chosen to keep a 1.5B-parameter model's memory footprint and per-step cost practical across the full nine-method, five-learning-rate, five-seed sweep on our hardware.

\paragraph{Data.} The Alpaca instruction dataset (\texttt{tatsu-lab/alpaca}), tokenized lazily per example via the model's own tokenizer, sequences capped at $512$ tokens (the original Alpaca fine-tuning recipe's own limit, and a bound on worst-case per-step cost under common-random-numbers batch reuse, Section~\ref{sec:selection}). Alpaca ships no canonical val/test split, so $5\%$/$5\%$ are carved out of the full set for validation/test (fixed split seed, shared identically by every method), leaving $90\%$ for training. Minibatch size is $2$ (this model's own reference training config). Each reported test loss is the loss, on a fixed $64$-example test sample, of the checkpoint that reached the lowest loss on a fixed $16$-example validation probe during that run.

\paragraph{Hyperparameters and protocol.} Table~\ref{tab:lora-hparams} lists the shared sweep for this section (Alpaca); Appendix~\ref{app:gsm8k} lists the one difference for GSM8K, its $3\times$ larger evaluation-cost budget. Every method is run at every (learning rate, seed) pair, and Section~\ref{sec:lora}'s numbers are reported at each method's own best learning rate, never a single shared one. The evaluation-cost budget is scaled down from MNIST's $B=3{,}000$ (Table~\ref{tab:mnist-hparams}) to $B=300$: a real forward+backward on this model costs on the order of seconds, not the low milliseconds of an MLP forward pass, so the larger budget would take weeks rather than hours across the full sweep; the accounting itself (Section~\ref{sec:setup-exp}) is unchanged.

\begin{table}[h]
\centering
\caption{Shared LoRA/Alpaca training protocol ($B=300$; GSM8K uses the same protocol except $B=900$, Appendix~\ref{app:gsm8k}).}
\label{tab:lora-hparams}
\begin{tabular}{ll}
\hline
Setting & Value \\
\hline
Evaluation-cost budget $B$ & $300$ (Section~\ref{sec:setup-exp} accounting) \\
Learning-rate grid & $\{10^{-4},3\times10^{-4},10^{-3},3\times10^{-3},10^{-2}\}$ \\
Seeds & $\{0,1,2,3,4\}$ (5 independent runs per cell) \\
Minibatch size & $2$ \\
Max sequence length & $512$ tokens \\
LoRA rank / alpha & $32$ / $64$ \\
LoRA target modules & query, value, output projections, MLP \\
Precision & bfloat16 (base + adapter) \\
Weight decay & $5\times10^{-4}$ (decoupled for AdamW; L2 for SGD+momentum and EvE's Adam fallback) \\
SGD momentum & $0.9$ \\
EvE population size $n_{\text{pop}}$ & $4$ (fixed, Section~\ref{sec:setup}) \\
EvE Adam-fallback steps $n_{\text{adam}}$ & $3$ (default, Section~\ref{sec:adam-fallback}) \\
EvE population initialization & anchor = LoRA's own zero-initialized adapter, not random (see below) \\
EvE polynomial mutation & off (Section~\ref{sec:mutation}) \\
\hline
\end{tabular}
\end{table}

\paragraph{Why EvE's population is initialized differently here.} LoRA's $B$ matrix is zero-initialized by convention, so a fresh adapter is an exact no-op on the frozen base model: at the very start of training, every method (including every baseline) sees the same loss the base model alone would produce. Section~\ref{sec:init}'s random anchor draw would overwrite that zero-init with noise before any training happens, immediately raising the starting loss far above every baseline's true starting point (measured: validation/test loss jumps to $\sim11$, near $\log(\text{vocab size})$, i.e.\ close to a fully untrained model). EvE's population is instead anchored at the adapter's own real starting values, with the other three population members still placed by the usual proximity jitter around that anchor (Section~\ref{sec:init}); this is the only initialization change made for this experiment, needed for a fair starting point rather than a tuning choice.

\paragraph{Detailed discussion of Figure~\ref{fig:lora}.} Panels (a) and (b) show the same training runs against the two axes. Against the shared evaluation-cost budget (a), EvE tracks below the Adam family and close to SGD, the same relative position it holds on MNIST (Figure~\ref{fig:mnist}a). Against wall-clock time (b), EvE has already exhausted its budget and stopped well before any baseline is halfway through, which panel (d) reports directly: $35.0\pm0.6$s for EvE against $58.8$--$70.7$s for the baselines, a $1.7$--$2.0\times$ wall-clock advantage for identical charged work. Panel (c) shows each method's learning-rate sensitivity on final test loss: EvE's best learning rate ($3\times10^{-4}$) is an order of magnitude higher than most of the Adam family's ($10^{-4}$; RAdam and Lookahead-Adam match EvE at $3\times10^{-4}$), consistent with most of its steps being DE moves rather than gradient steps, so its occasional Adam-fallback steps need a larger nominal rate. Panel (e) shows the final test-loss picture: EvE's median sits above the tight Adam-family cluster ($1.240$ versus $1.115$--$1.124$, roughly $11\%$ higher), and its min--max range is wider than any baseline's, the run-to-run variance a four-member population carries relative to a single gradient trajectory (Section~\ref{sec:limitations}).

\section{GSM8K: Dataset and Detailed Analysis}
\label{app:gsm8k}

\subsection{Same-Task Fine-Tuning: LoRA on GSM8K}
\label{sec:gsm8k}

Section~\ref{sec:lora} trains on Alpaca and reports loss on its own held-out split, which only asks whether an optimizer fits its training distribution, not whether the resulting model is any good at a task someone cares about. We therefore fine-tune directly on GSM8K \citep{cobbe2021training}, raising the evaluation-cost budget to $900$ ($3\times$ Alpaca's $300$).

EvE shows the same speed/quality trade-off here as on Alpaca (Section~\ref{sec:lora}), now at three times the budget and on a genuinely different training target (multi-step reasoning chains rather than short instruction-following answers), detailed panel-by-panel below alongside a downstream accuracy check on the full test set (Table~\ref{tab:gsm8k-full}).

\begin{figure}[t]
\centering
\scalebox{0.65}{%
\begin{minipage}{\linewidth}
\begin{subfigure}[b]{0.48\linewidth}
\includegraphics[width=\linewidth]{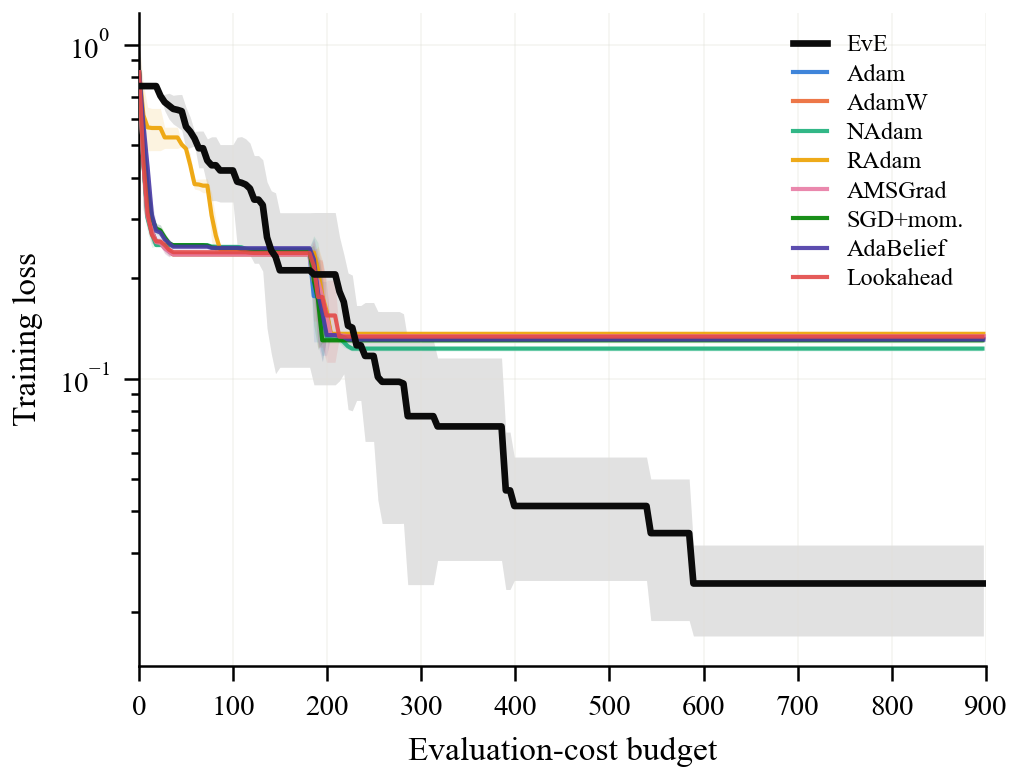}
\caption{Training loss vs.\ evaluation-cost budget.}
\end{subfigure}\hfill
\begin{subfigure}[b]{0.48\linewidth}
\includegraphics[width=\linewidth]{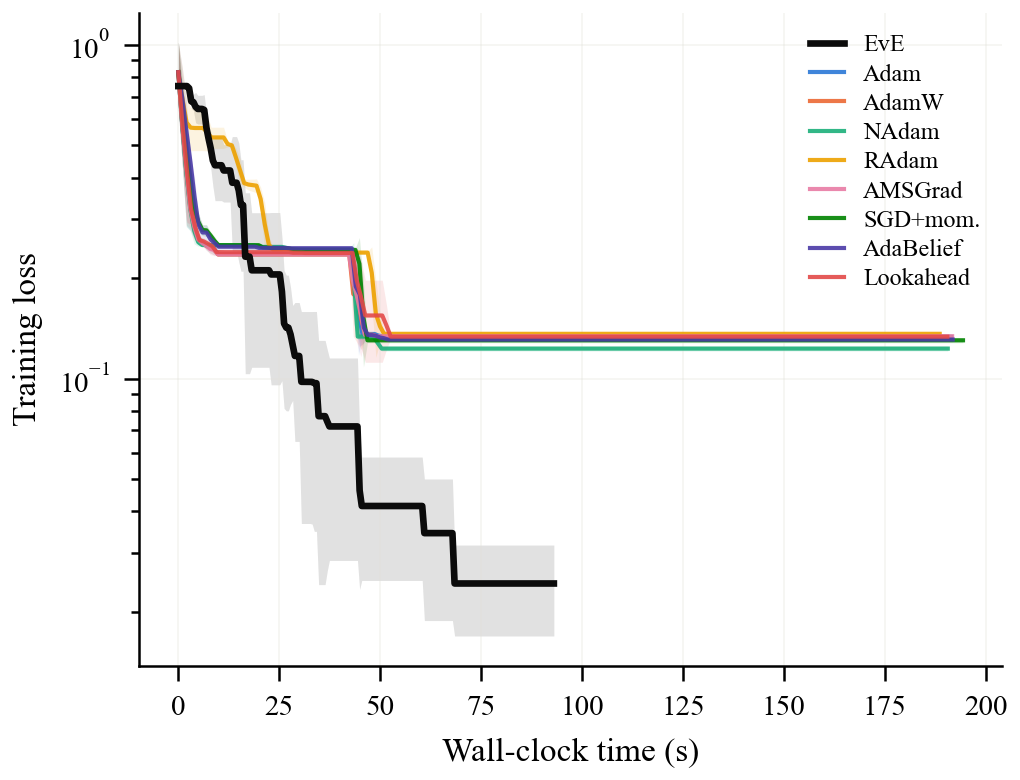}
\caption{Training loss vs.\ wall-clock time.}
\end{subfigure}\\[2pt]
\begin{subfigure}[b]{0.32\linewidth}
\includegraphics[width=\linewidth]{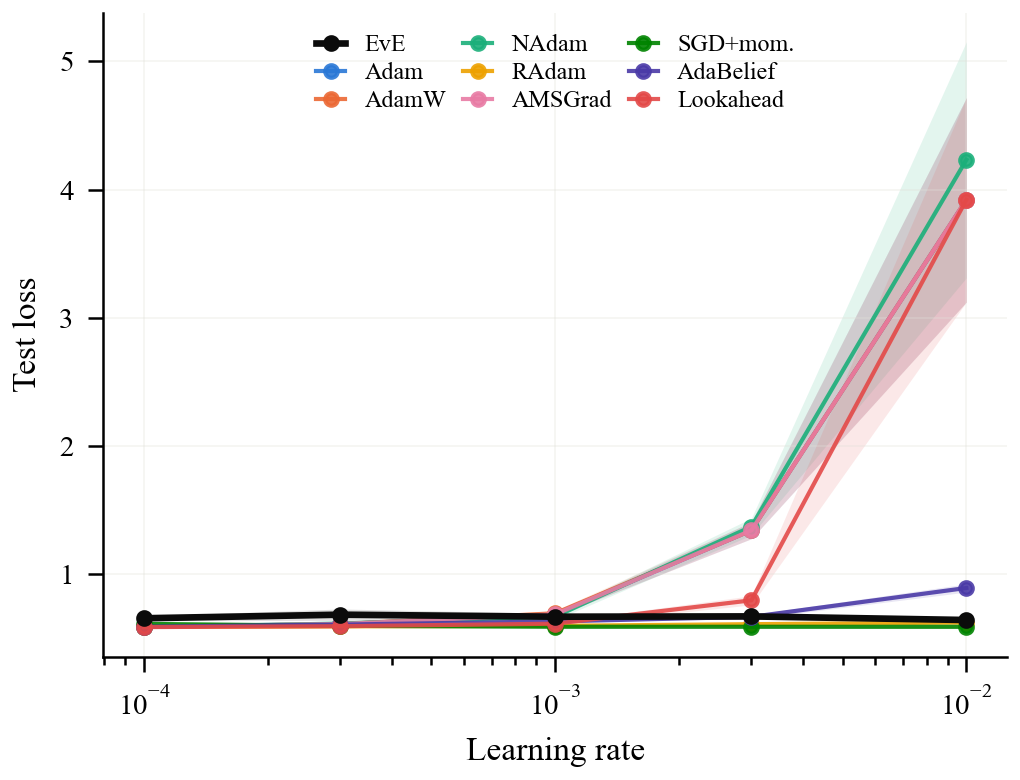}
\caption{Test loss vs.\ learning rate.}
\end{subfigure}\hfill
\begin{subfigure}[b]{0.32\linewidth}
\includegraphics[width=\linewidth]{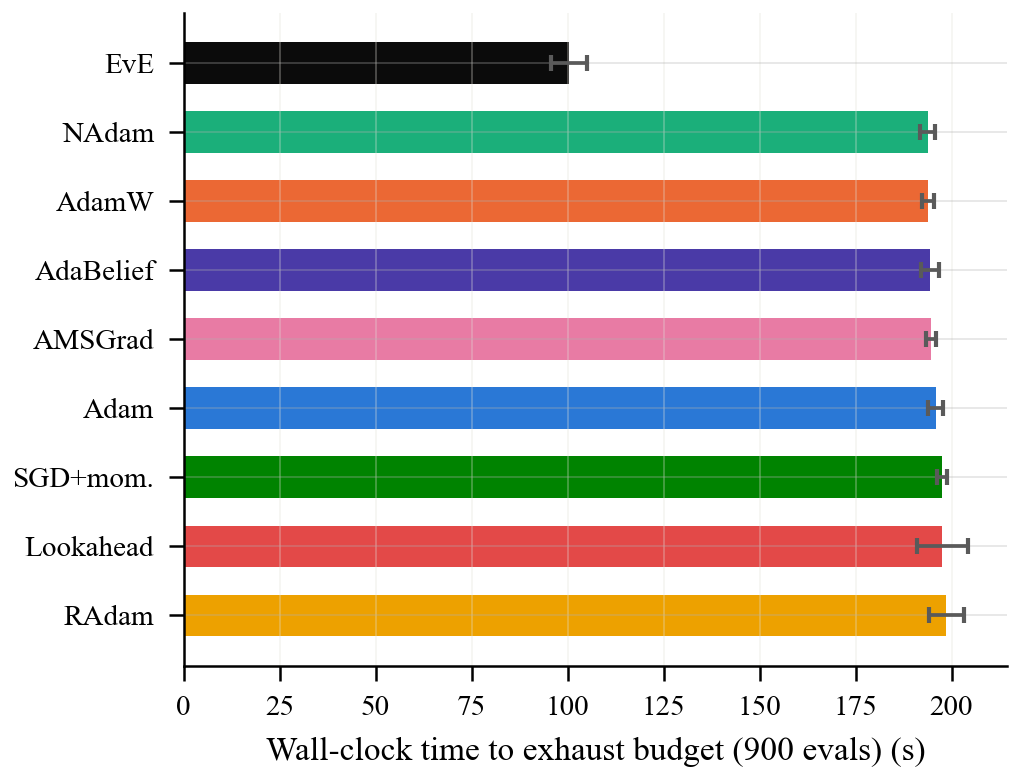}
\caption{Wall-clock time to exhaust the shared budget.}
\end{subfigure}\hfill
\begin{subfigure}[b]{0.32\linewidth}
\includegraphics[width=\linewidth]{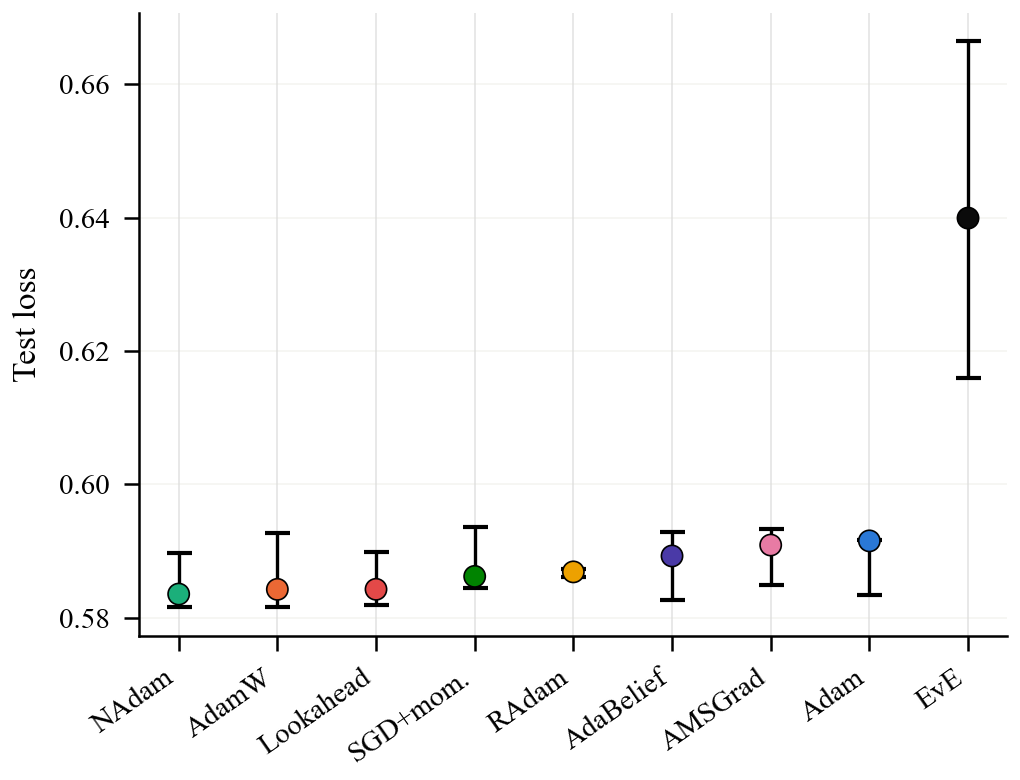}
\caption{Best-LR test loss (median, min-max over 5 seeds).}
\end{subfigure}
\end{minipage}}
\caption{EvE vs.\ eight baselines fine-tuning Qwen2.5-1.5B-Instruct with LoRA directly on GSM8K, under a shared evaluation-cost budget of $900$ ($3\times$ Figure~\ref{fig:lora}'s Alpaca budget), at each method's own best learning rate over 5 seeds. (a)/(b) training runs vs.\ budget and vs.\ wall-clock time. (c) learning-rate sensitivity. (d) wall-clock cost of the shared budget, sorted. (e) final test loss reached (lower is better).}
\label{fig:gsm8k}
\end{figure}

\paragraph{Why train directly on GSM8K?} Section~\ref{sec:lora}'s Alpaca fine-tuning and Appendix~\ref{app:lora}'s protocol only measure whether an optimizer fits its own training distribution; neither says anything about a downstream task an Alpaca-tuned model was never asked to solve. Testing such a model zero-shot on GSM8K \citep{cobbe2021training} would conflate two different questions into one number, with no way to tell which one moved: whether the optimizer learned the fine-tuning task well, and whether fine-tuning damaged some unrelated capability the base model already had (catastrophic forgetting). Training and evaluating on the same task removes that confound: any gap between methods here reflects how well each optimizer fits GSM8K itself, nothing else.

\paragraph{Data and masking.} GSM8K ships its own train ($7{,}473$ examples) and test ($1{,}319$ examples) splits; unlike Alpaca (Appendix~\ref{app:lora}), no test carve-out is needed, only a validation split out of train (same $\text{VAL\_FRACTION}$/split-seed convention as every other experiment here). Each training target matches the exact prompt format used to build the answer: solve step by step, ending with the final number after \texttt{"\#\#\#\# "}. GSM8K's own answer field already is step-by-step reasoning in this form, with only its \texttt{<<calculator annotations>>} stripped (the evaluation prompt was never asked to produce those, so training on them would be a train/eval mismatch for no benefit). Standard SFT masking applies: the question (user turn) is excluded from the loss, the answer (assistant turn) is trained on, verified directly against the tokenizer's own masked-flag convention rather than assumed. Model, LoRA configuration, precision, sequence-length cap, learning-rate grid, seeds, and EvE's population initialization (Table~\ref{tab:lora-hparams}'s anchor-at-the-adapter's-own-zero-init, Appendix~\ref{app:lora}, since the LoRA architecture's zero-initialized $B$ matrix is a property of the model, not of Alpaca specifically) are all identical to Section~\ref{sec:lora}; the dataset differs, and the evaluation-cost budget is raised to $900$ ($3\times$ Section~\ref{sec:lora}'s $300$), so that fine-tuning has enough room to say something about downstream reasoning capability rather than only about convergence speed at a matched, short budget. This is a deliberate, stated departure from the matched-budget protocol used for every other comparison in this paper (Section~\ref{sec:setup-exp}); it means Figure~\ref{fig:gsm8k}'s wall-clock numbers are not directly comparable to Figure~\ref{fig:lora}'s in absolute terms, only in their relative (speedup/quality-gap) pattern.

\paragraph{Detailed discussion of Figure~\ref{fig:gsm8k}.} The overall pattern matches Figure~\ref{fig:lora}, at three times the per-method budget. Panels (a) and (b) again show EvE below the Adam family against the budget, with panel (d) converting the wall-clock cost of that larger budget into seconds: $100.4\pm4.7$s for EvE against $193.8$--$198.7$s for the baselines, a $1.9$--$2.0\times$ advantage, essentially the same ratio as Section~\ref{sec:lora}'s despite three times the work. Panel (c) shows EvE's best learning rate by test loss ($10^{-2}$) is two orders of magnitude above seven of the eight baselines' ($10^{-4}$); chosen on validation loss instead, EvE's rate is $10^{-4}$, like theirs (Appendix~\ref{app:lr-selection}). At this larger budget, SGD with momentum's own best rate also moves up to $10^{-2}$, matching EvE rather than sitting an order below it as at the smaller Alpaca/GSM8K budgets, consistent with more evaluations giving every method's occasional larger step more room to pay off before the budget runs out. Panel (e) shows the final test-loss picture: EvE's median sits above the baselines' tight cluster ($0.640$ versus $0.585$--$0.590$, roughly $9\%$ higher, a smaller relative gap than Section~\ref{sec:lora}'s $11\%$; this smaller gap is not robust to the selection rule, since choosing EvE's rate on validation loss gives $0.654$, about $11\%$ higher, Appendix~\ref{app:lr-selection}), with a wider min--max range than any baseline (test-loss standard deviation $0.016$ against every baseline's $\le0.005$), the same run-to-run variance pattern as everywhere else in this paper (Section~\ref{sec:limitations}) though less extreme in relative terms than at the smaller budget.

\paragraph{Downstream reasoning accuracy.} Test loss (Figure~\ref{fig:gsm8k}e) measures whether a method fits GSM8K's own answer distribution; it does not by itself confirm that a lower-loss checkpoint solves more problems correctly. We therefore also scored downstream exact-match accuracy, with greedy (temperature-$0$) decoding and at most $512$ generated tokens, on GSM8K's $1{,}319$-question test split. This check has a practical purpose beyond ranking optimizers: it asks whether this fine-tuning recipe is worth running at all, and a practitioner can answer that with EvE in about half the wall-clock time Adam needs ($100.4\pm4.7$s against $193.8$--$198.7$s per run at $900$ evaluations, Figure~\ref{fig:gsm8k}d), instead of spending twice the compute on Adam to reach the same conclusion. The same holds at every longer budget we ran (Table~\ref{tab:gsm8k-full}), so the saving comes from EvE's cheaper charged evaluations, not from any one budget. This is the sense in which we call EvE a proxy: the signal a search loop needs arrives sooner, while the final accuracy of an EvE-trained checkpoint is lower than Adam's.

An initial probe on a fixed $100$-question subset of the test split showed every fine-tuned checkpoint scoring below the untouched base model. To test whether a longer training exposure would change that, we extended each method's best configuration (one learning rate and seed per method, the same at every budget) to $10{,}000$ and then $30{,}000$ evaluations and re-scored on the same subset; longer training did not restore accuracy, which motivated scoring the whole test split. The probe cost minutes per checkpoint and the full evaluation roughly a day for all checkpoints, so the cheap probe decided whether the expensive evaluation was warranted. We report two scorings, both with no hand correction: strict, where only the model's own final \texttt{"\#\#\#\# N"} answer counts, and flexible, which takes that answer if present, otherwise a boxed answer, otherwise the last number in the completion. The base model rarely uses our answer format, so strict scoring understates what it can do; flexible scoring is the fair comparison of capability, and strict is reported alongside it so the difference is visible.

On the full $1{,}319$-question test set (every checkpoint scored on the same questions; Table~\ref{tab:gsm8k-full}), the base model scores $27.9\%$ under strict scoring and $73.5\%$ under flexible scoring; the gap arises because it rarely writes the \texttt{"\#\#\#\# N"} format. The fine-tuned checkpoints do use that format, so their two scores nearly coincide. Strict scoring alone would therefore misleadingly favor the fine-tuned checkpoints over the base model, purely because of format. Under flexible scoring, the fair comparison, the base model beats every fine-tuned checkpoint by at least $14$ points (paired exact McNemar tests, all $p<10^{-4}$), so fine-tuning under this recipe lowers accuracy for both optimizers. Between the two optimizers, Adam is significantly more accurate than EvE at $900$ evaluations ($59.4\%$ against $54.4\%$, $p=0.0007$) and at $10{,}000$ ($58.2\%$ against $53.7\%$, $p=0.002$), while at $30{,}000$ the difference is not significant ($52.6\%$ against $54.4\%$, $p=0.24$), because Adam's accuracy declines as its budget grows ($59.4\%$ to $52.6\%$, $p<10^{-4}$) while EvE's does not change. EvE's $30{,}000$-evaluation adapter is identical to its $900$-evaluation adapter: training retains the checkpoint with the lowest validation-probe loss, and no later checkpoint beat one reached early in training, so EvE contributes two distinct budgets here, not three. Each budget uses a single training run per method, unlike the five-seed protocol of the rest of this paper.

\begin{table}[h]
\centering
\caption{GSM8K downstream exact-match accuracy on the full $1{,}319$-question test set, one training run per row, with the wall-clock training time of that run in seconds. The interval is the $95\%$ Wilson interval on flexible accuracy. EvE's $30{,}000$-evaluation adapter is identical to its $900$-evaluation adapter, so its accuracy is shown once, with both training times.}
\label{tab:gsm8k-full}
\begin{tabular}{lcccc}
\hline
Checkpoint & Strict (\texttt{"\#\#\#\# N"}) & Flexible & Flexible interval & Training time (s) \\
\hline
Base model (no fine-tuning) & $27.9\%$ & $73.5\%$ & $[71.1,\,75.9]$ & -- \\
EvE, $900$ / $30{,}000$ evaluations & $54.4\%$ & $54.4\%$ & $[51.7,\,57.1]$ & $105$ / $3{,}396$ \\
EvE, $10{,}000$ evaluations & $53.7\%$ & $53.7\%$ & $[51.0,\,56.4]$ & $1{,}104$ \\
Adam, $900$ evaluations & $59.4\%$ & $59.4\%$ & $[56.7,\,62.0]$ & $196$ \\
Adam, $10{,}000$ evaluations & $58.2\%$ & $58.2\%$ & $[55.5,\,60.9]$ & $2{,}093$ \\
Adam, $30{,}000$ evaluations & $52.5\%$ & $52.6\%$ & $[49.9,\,55.3]$ & $6{,}284$ \\
\hline
\end{tabular}
\end{table}

This finding does not change this paper's central claim, which rests on the loss-based, evaluation-cost-matched comparison of Figure~\ref{fig:gsm8k} and Section~\ref{sec:gsm8k}. Narrow supervised fine-tuning degrading a pretrained capability is a documented phenomenon \citep{luo2023empirical}, and because each method's ``best learning rate'' (Appendix~\ref{app:lora}) is selected by held-out test loss rather than downstream accuracy, a checkpoint can lower that token-level loss by imitating the surface form of the \texttt{"\#\#\#\# N"} answer key without improving arithmetic. This is reported because a claim this paper does not make (that GSM8K fine-tuning improves downstream accuracy for any method, including EvE) should not be left ambiguous by omission.

\end{document}